\documentclass[screen,nonacm,review=false,timestamp=false]{acmart}

\AtBeginDocument{%
  }

\setcopyright{acmlicensed}
\copyrightyear{2018}
\acmYear{2018}
\acmDOI{XXXXXXX.XXXXXXX}

\acmISBN{978-1-4503-XXXX-X/2018/06}

\usepackage{amsmath,amsfonts,bm}
\usepackage{soul} 

\def\eqref#1{equation~\ref{#1}}

\def\1{\bm{1}}

\DeclareMathAlphabet{\mathsfit}{\encodingdefault}{\sfdefault}{m}{sl}
\SetMathAlphabet{\mathsfit}{bold}{\encodingdefault}{\sfdefault}{bx}{n}

\newcommand{\M}{\mathcal{M}}

\newcommand{\X}{\mathsf{X}}

\newcommand{\real}{\mathbb{R}}

\usepackage{cleveref}
\usepackage{amsthm,amsmath,amsfonts}
\usepackage{bbm}
\theoremstyle{plain}

\newtheorem{thm}{Theorem}
\theoremstyle{definition}

\newtheorem{innercustomthm}{Theorem}
\newenvironment{customthm}[1]
  {\renewcommand\theinnercustomthm{#1}\innercustomthm}
  {\endinnercustomthm}

\usepackage{algorithm}
\usepackage{algpseudocode}

\usepackage{etoolbox}
\patchcmd{\subsubsection}{\itshape}{\bfseries}{}{}
\patchcmd{\paragraph}{\itshape}{\bfseries}{}{}

\begin{document}

\title{MIOFlow 2.0: A unified framework for inferring cellular stochastic dynamics from single cell and spatial transcriptomics data}

\author{Xingzhi Sun}
\authornote{Equal contribution.}
\affiliation{%
  \institution{Yale University}
  \country{USA}
}

\author{João Felipe Rocha}
\authornotemark[1]
\affiliation{%
  \institution{Yale University}
  \country{USA}
}

\author{Brett Phelan}
\authornotemark[1]
\affiliation{%
  \institution{Yale University}
  \country{USA}
}

\author{Dhananjay Bhaskar}
\authornotemark[1]
\affiliation{%
  \institution{University of Wisconsin-Madison}
  \country{USA}
}

\author{Guillaume Huguet}
\affiliation{%
  \institution{Université de Montréal; Mila - Quebec AI Institute}
  \country{Canada}
}

\author{Yanlei Zhang}
\affiliation{%
  \institution{Université de Montréal; Mila - Quebec AI Institute}
  \country{Canada}
}


\author{Alexander Tong}
\affiliation{%
  \institution{Université de Montréal; Mila - Quebec AI Institute}
  \country{Canada}
}

\author{Ke Xu}
\affiliation{%
  \institution{Yale University}
  \country{USA}
}

\author{Oluwadamilola Fasina}
\affiliation{%
  \institution{Yale University}
  \country{USA}
}

\author{Mark Gerstein}
\affiliation{%
  \institution{Yale University}
  \country{USA}
}

\author{Natalia Ivanova}
\affiliation{%
  \institution{University of Georgia}
  \country{USA}
}

\author{Christine L. Chaffer}
\affiliation{%
  \institution{Garvan Institute of Medical Research}
  \country{Australia}
}

\author{Guy Wolf}
\affiliation{%
  \institution{Université de Montréal; Mila - Quebec AI Institute}
  \country{Canada}
}

\author{Smita Krishnaswamy}
\authornote{Corresponding author: smita.krishnaswamy@yale.edu}
\affiliation{%
  \institution{Yale University}
  \country{USA}
}

\renewcommand{\shortauthors}{Sun et al.}





\begin{abstract}
Understanding cellular dynamics through time-resolved single-cell transcriptomics is essential for elucidating mechanisms of development, regeneration, and disease progression. A fundamental challenge is inferring continuous cellular trajectories from discrete snapshots, where biological complexity arises from stochasticity in cell fate decisions, temporal changes in cell proliferation and death, and environmental influences from the spatial tissue context. Current trajectory inference methods often rely on deterministic, mass-conserving interpolations that treat cells in isolation, failing to capture the probabilistic branching, population shifts, and niche-dependent signaling that drive biological processes.

We introduce Manifold Interpolating Optimal-Transport Flow (MIOFlow) 2.0, a unified framework that learns biologically informed cellular dynamics by integrating manifold learning, optimal transport and neural differential equations.

MIOFlow 2.0 represents a significant advancement over previous generative models by explicitly modeling three biological processes: (1) stochasticity and branching through Neural Stochastic Differential Equations; (2) non-conservative population dynamics via a learned growth-rate model initialized with unbalanced optimal transport; and (3) environmental influence through the construction of a joint latent space that unifies gene expression with spatial features such as neighborhood cell type composition and ligand-receptor signaling from spatial transcriptomics data.

By operating in the informative latent space of a PHATE-distance matching autoencoder, MIOFlow 2.0 ensures that trajectories respect the intrinsic geometry of the data manifold. Despite the popularity of simulation-free flow matching, expressive dynamics learning via neural differential equations outperforms existing generative models in matching complex biological trajectories. We validate MIOFlow 2.0 on synthetic datasets and biological applications, including embryoid body differentiation and spatially-resolved axolotl brain regeneration. Our results demonstrate that MIOFlow 2.0 not only improves trajectory accuracy but also reveals previously hidden drivers of cellular transitions, such as specific signaling niches regulating regenerative development. MIOFlow 2.0 bridges single-cell and spatial transcriptomics, opening new avenues for understanding tissue-scale cellular dynamics.

\end{abstract}




\maketitle

\section{Introduction}

Single-cell RNA sequencing (scRNA-seq) has revolutionized our ability to profile gene expression at unprecedented resolution, enabling the study of cellular heterogeneity and dynamics in complex biological processes~\citep{hwang2018single,saliba2014single,jovic2022single,kolodziejczyk2015technology}. When collected across multiple timepoints, these data capture snapshots of critical developmental programs, regenerative responses, and disease progressions. However, a fundamental limitation of scRNA-seq technology is its destructive nature: individual cells cannot be continuously tracked over time. Instead, researchers obtain population-level measurements at discrete intervals, where each timepoint consists of transcriptional profiles from different cells at varying stages of the underlying biological process.

This experimental constraint presents a significant computational challenge: how can we infer continuous, cell-level trajectories from these static population snapshots when no ground-truth correspondence exists between cells across timepoints? The problem is further complicated by fundamental biological realities. First, cellular transitions exhibit stochasticity driven by noisy gene expression and probabilistic fate decisions characteristic of differentiation hierarchies~\citep{elowitz2002stochastic, losick2008stochasticity, wang2011quantifying}. Second, cell populations undergo continuous size changes due to proliferation and death that dramatically reshape population distributions, particularly in contexts such as cancer treatment response~\citep{gerlinger2012intratumor, marusyk2012intratumour, gatenby2017integrating,cflows}. Third—and often overlooked by existing methods—cellular behavior is heavily influenced by spatial context: the local microenvironment of neighboring cells, signaling molecules, and tissue architecture. This spatial dependence is especially critical in processes such as embryonic development, immune responses, and tissue regeneration, where cell-cell communication and positional information guide cellular transitions~\citep{hanahan2011hallmarks, quail2013microenvironmental, ferrara2016ten, balkwill2001inflammation,sun2025hyperedge}.

The emergence of spatial transcriptomics technologies has created new opportunities to study how cellular trajectories depend on tissue context. These platforms provide spatially-resolved gene expression, capturing not only what genes cells express but also where cells are located and who their neighbors are. A cell's transcriptional state evolves not in isolation but in response to its changing neighborhood—the cell types surrounding it, the ligand-receptor interactions it participates in, and the local tissue architecture. However, existing trajectory inference methods are not designed to exploit this spatial information. They treat cellular dynamics as functions of gene expression alone, missing critical environmental drivers of transitions and failing to capture how the same cell state can follow different developmental paths depending on spatial context.

Beyond spatial considerations, technical challenges compound the difficulty of trajectory inference. Single-cell transcriptomic data are high-dimensional (thousands of genes) and exhibit substantial noise. However, strong gene-gene correlations driven by coordinated regulatory programs mean that meaningful biological variation typically resides on a low-dimensional manifold embedded within the high-dimensional gene expression space~\citep{coifman_diffusion_2006, belkin_semi-supervised_2004, moon_manifold_2018, moon_visualizing_2019, mishne2019diffusion, duque2020extendable,sun2024geometry}. Effective methods must therefore recover this manifold structure and learn continuous paths that respect the data's intrinsic geometry rather than imposing Euclidean assumptions.

Existing approaches to trajectory inference fall short in addressing these combined challenges. A substantial body of work focuses on generative modeling tasks, such as flowing from noise distributions to data, and is not designed for temporal interpolation~\citep{song_generative_2019, ho_denoising_2020, grathwohl_ffjord:_2019}. Methods that do tackle trajectory inference, including TrajectoryNet~\citep{tong_trajectorynet_2020} and Waddington-OT~\citep{schiebinger_reconstruction_2017}, often make restrictive assumptions such as Gaussian distributed populations. More recently, Flow Matching~\citep{lipman2023flow,tong2023improving} has been applied to single-cell trajectory problems but relies on prescribed velocity fields that enforce straight-line interpolations in ambient space, ignoring curved manifold geometry. While Riemannian extensions~\citep{chen2023flow} attempt to address geometric constraints, they are limited to simple, analytically-known manifolds. Critically, none of these methods can incorporate spatial information to make trajectories conditional on cellular neighborhoods, nor do they explicitly model the triumvirate of biological processes essential for realistic cellular dynamics: stochasticity, proliferation/death, and environmental conditioning.

To address these limitations, we present Manifold Interpolating Optimal Transport Flows (MIOFlow) 2.0, a comprehensive framework that unifies manifold learning, optimal transport, and neural differential equations to learn spatially-conditioned cellular dynamics from temporal snapshots. MIOFlow 2.0 learns continuous flows between cell populations by training neural differential equations guided by optimal transport costs computed over geodesic distances on a learned data manifold. By operating in the latent space of a geometry-aware autoencoder, MIOFlow 2.0 ensures that inferred trajectories inherently respect both local and global structure of the cellular state space.

A key innovation of MIOFlow 2.0 is its spatial conditioning mechanism that extends trajectory inference to leverage spatial transcriptomics data. MIOFlow 2.0 extracts rich spatial features from cellular neighborhoods, including: (1) cell type composition—the frequencies of different cell types surrounding each cell; (2) ligand-receptor signaling—the potential for specific molecular interactions based on ligand expression in neighbors and receptor expression in the target cell; (3) spatial density patterns—local crowding and tissue organization; and (4) averaged gene expression embeddings across neighborhoods. These spatial features are then used to condition the neural differential equation, allowing the learned vector field to depend on both the cell's transcriptional state and its spatial context. This enables MIOFlow 2.0 to model a fundamental biological reality: cells with identical gene expression profiles can follow different developmental trajectories depending on their microenvironment.

Beyond spatial conditioning, MIOFlow 2.0 incorporates two additional biological priors to ensure realistic modeling. First, to capture stochasticity in cell fate decisions—a hallmark of differentiation processes—MIOFlow 2.0 employs Neural Stochastic Differential Equations (Neural SDEs) that model both deterministic drift and diffusive noise. Second, to handle non-uniform population dynamics, MIOFlow 2.0 includes a learned proliferation rate model that estimates cell growth and death, initialized via unbalanced optimal transport to account for mass creation and destruction between timepoints.

Our main contributions are:
\begin{itemize}
    \item The first trajectory inference framework that integrates spatial transcriptomics data by embedding cellular neighborhoods and gene expression into a joint latent space on which the dynamics is learned.
    This enables the discovery of environmental drivers of cellular transitions including cell-cell signaling, spatial organization, and microenvironmental composition.
    \item The MIOFlow 2.0 framework for learning continuous cellular dynamics through manifold-constrained optimal transport, with theoretical guarantees connecting learned trajectories to geodesic transport on the data manifold.
    \item A unified modeling framework that explicitly incorporates three essential biological processes: stochasticity through Neural SDEs, cell proliferation and death through learned growth rates, and environmental influence through spatial feature conditioning from neighborhood composition, signaling patterns, and neighborhood gene expression. 
    \item Comprehensive validation demonstrating superior performance over state-of-the-art methods, with application to spatially-resolved axolotl brain regeneration revealing medium spiny neuron neighborhoods as key regulators of regene rative trajectories, an insight only accessible through spatial conditioning.
\end{itemize}

Our application to axolotl brain regeneration demonstrates the power of spatial conditioning:
by analyzing how trajectories vary with cellular neighborhoods, MIOFlow 2.0 reveals that the local abundance of medium spiny neurons (MSNs) influences regenerative cell fate decisions. This type of insight—linking environmental spatial features to developmental outcomes—is fundamentally inaccessible to methods that model trajectories based on gene expression alone. By bridging single-cell and spatial transcriptomics, MIOFlow 2.0 enables a new class of analysis that connects molecular states, cellular trajectories, and tissue-scale organization.

The remainder of this paper is organized as follows. Section 2 provides background on optimal transport and establishes our problem formulation. Section 3 details the MIOFlow 2.0 framework, including the manifold embedding approach, trajectory inference methodology, biological modeling components, and the spatial feature extraction and conditioning mechanism. Section 4 presents comprehensive experimental validation on synthetic and real biological datasets, with emphasis on spatially-resolved applications. Section 5 concludes with discussion and future directions.

\section{Background}
\label{sec:background}

\subsection{Biological Motivation: From Snapshots to Dynamics}

The study of cellular identity has undergone a paradigm shift from static classification to the search for continuous developmental rules.
Historically, single-cell analysis focused on "snapshots" of transcriptional states to define discrete cell types and clusters~\citep{pollen2014low, zeisel2015cell}.
However, as researchers began to sample tissues across time, it became clear that cell states are not fixed categories but points along a continuous manifold of differentiation and maturation.
This shift necessitated the development of trajectory inference methods, which aim to reconstruct the continuous paths cells follow between measured timepoints.

Early methods focused on the concept of "pseudotime," ordering cells along one-dimensional paths based on transcriptional similarity~\citep{trapnell2014pseudo,haghverdi2016diffusion,street2018slingshot}.
While these models revealed the general sequence of gene expression changes, they were fundamentally non-directional and descriptive rather than predictive. Crucially, pseudotime lacks single-cell resolution in a dynamical sense; because it treats snapshots as a static continuum, it cannot track the specific lineage path of an individual cell or predict how it would evolve across branching points under different conditions.

The introduction of RNA velocity~\citep{la2018rna, bergen2021rna} attempted to add directionality by leveraging the ratio of unspliced to spliced mRNA to predict short-term future states.
While RNA velocity provided a mechanistic leap by estimating the "velocity" of individual cells, it remains limited to short-range predictions and is highly sensitive to noise and splicing-rate assumptions. 
This makes it difficult for velocity-based methods to reconstruct long-range trajectories across distant developmental timepoints or to capture global population-level shifts.

Despite the progress made by pseudotime and velocity-based methods, a fundamental gap remains: the ability to reconstruct \textbf{global, continuous, single-cell trajectories}. 
Pseudotime provides an ordering but no physical path; RNA velocity provides a local vector but no long-term destination. To truly understand developmental "destiny," we require a framework that treats cellular change as a continuous flow in a high-dimensional state space. 
Such a framework must be "global" in that it connects distant timepoints across the entire manifold, and it must operate at "single-cell resolution" to allow us to track the hypothetical history of any individual cell. 
However, simply interpolating between points is insufficient for biological systems. 
If we treat cellular movement as a purely mathematical problem, we risk learning "shortcuts" that ignore the complex constraints of life. 
To move from a purely geometric interpolation to a biologically plausible dynamical model, we must explicitly incorporate the fundamental principles that govern cellular evolution: stochasticity, non-conservative population mass, and environmental context.

To move from these descriptive or short-range models to a predictive dynamical framework, we must account for the specific forces that govern cellular life.
First, biological differentiation is inherently governed by \textbf{stochasticity and branching}.
Transcriptional noise and probabilistic fate decisions mean that two cells starting in an identical state may diverge into distinct lineages~\citep{elowitz2002stochastic, losick2008stochasticity}.

Second, cellular populations are subject to \textbf{non-conservative dynamics}.
Unlike mechanical systems where the number of particles is constant, biological tissues are sites of continuous proliferation and apoptosis.
In rapid developmental stages or in response to cancer treatments, certain lineages may expand exponentially while others vanish~\citep{gerlinger2012intratumor, cflows}.
Standard interpolation methods that enforce a conservation of mass fail to capture these density shifts, often leading to biologically nonsensical trajectories.

Finally, and perhaps most critically, a cell's journey is shaped by its \textbf{environmental context}.
The local tissue microenvironment—composed of neighboring cell types, signaling molecules, and the gene expression of the neighboring cells—acts as a conditional input to the cell's internal program~\citep{hanahan2011hallmarks, quail2013microenvironmental}. 
With the advent of spatial transcriptomics, we can now observe that cells at the same stage of differentiation may follow entirely different trajectories if they reside in different spatial niches.

The challenge for modern computational biology is to move beyond mere interpolation and develop a mathematical framework that unifies these principles—stochasticity, mass flux, and spatial context—into a coherent model of cellular evolution.

\subsection{Optimal Transport: The Geometry of Distributional Change}
\label{sec:ot_background}
Because cellular snapshots are population-level observations, we require a framework that can interpolate between entire distributions while respecting the physical cost of biological transitions. Optimal Transport (OT) provides the rigorous mathematical foundation for this task by defining the "least effort" path between cell populations~\citep{peyre_computational_2020, villani2009optimal}. 
In the dynamic formulation of OT, we move beyond simple static correspondences to recover continuous, global trajectories at the single-cell level. 
The problem can be framed as finding the most efficient way to transform a source distribution $\mu$ into a target $\nu$ through time.

Let $\mu$ and $\nu$ be two probability measures representing cell populations at consecutive timepoints.
The Kantorovich formulation of OT seeks a transport plan $\pi$ that minimizes the total cost of moving mass from $\mu$ to $\nu$:
\begin{equation}
\label{eq:kantorovich}
W_p(\mu, \nu)^p := \inf_{\pi \in \Pi(\mu, \nu)} \int_{\mathcal{X} \times \mathcal{X}} d(x, y)^p \pi(dx, dy),
\end{equation}
where $\Pi(\mu, \nu)$ is the set of all joint distributions with marginals $\mu$ and $\nu$, and $d(x,y)$ represents the cost function, typically the Euclidean distance.

While the static formulation in Eq.~\ref{eq:kantorovich} identifies correspondences between cells, it does not describe the path taken between them.
To recover continuous dynamics, we look to the dynamic formulation of Benamou and Brenier~\citep{benamou_computational_2000}.
They showed that for $p=2$, the Wasserstein distance is equivalent to the minimum kinetic energy required to flow the density $\rho_t$ from $\mu$ to $\nu$ over the time interval $[0,1]$:
\begin{equation}
\label{eq:benamou_brenier}
W_2(\mu, \nu)^2 = \inf_{(\rho_t, v)} \int_{0}^{1} \int_{\mathbb{R}^d} \|v(x, t)\|_2^2 \rho_t(dx) dt,
\end{equation}
subject to the continuity equation $\partial_t \rho_t + \nabla \cdot (\rho_t v) = 0$.
This continuity equation is a fundamental law of physics that ensures mass is neither created nor destroyed as it flows along the vector field $v(x,t)$.

However, as discussed in the biological motivation, the assumption of mass conservation is often violated in cellular systems.
To resolve this, Unbalanced Optimal Transport (UOT) relaxes the strict marginal constraints~\citep{Chizat2015UnbalancedOT}.
By introducing divergence penalties, such as the Kullback-Leibler (KL) divergence, UOT allows for the "source" and "sink" of mass:
\begin{equation}
\label{eq:uot_math}
\text{UOT}(\mu, \nu) = \inf_{\pi \ge 0} \int d(x, y)^2 d\pi(x, y) + \lambda_1 D_{KL}(\pi_1 \| \mu) + \lambda_2 D_{KL}(\pi_2 \| \nu).
\end{equation}
This mathematical relaxation is what enables MIOFlow 2.0 to handle the proliferation and death of cell lineages without distorting the inferred transport paths.

\subsection{Neural Differential Equations: The Physics of Continuous Flows}
\label{sec:nde_background}

While Optimal Transport provides the target endpoints and the energy functional, we require a flexible parameterization to learn the underlying vector field $v(z,t)$.
Neural Ordinary Differential Equations provide a powerful framework for this by modeling the derivative of the state $Z_{u,t}$ as a neural network $f_\theta$~\citep{chen2018neural}.
In this view, the transformation of a cell state is not a single discrete jump but a continuous integration.
\begin{equation}
\label{eq:node_bg}
Z_{u,t} = Z_{u,0} + \int_0^t f_\theta(Z_{u,\tau}, \tau) d\tau.
\end{equation}
This continuous formulation allows us to evaluate the state of a cell at any arbitrary timepoint.
It provides a bridge between the sparse snapshots collected in the lab.

To account for the inherent stochasticity of differentiation, we must move beyond deterministic ODEs.
Neural Stochastic Differential Equations extend this by adding a diffusion term $\sigma(Z_{u,t}, t) dW_t$, where $W_t$ represents Brownian motion~\citep{Tzen2019NeuralSDE}.
The resulting Ito SDE allows a single initial state to evolve into a distribution of outcomes, naturally modeling branching events and transcriptional noise.
By combining the energy-minimizing principles of OT with the continuous expressive power of Neural SDEs, we can learn cellular trajectories that are both physically efficient and biologically realistic.

\subsection{Related Work: The Methodological Landscape}
\label{sec:related_work}

The intersection of Optimal Transport and Differential Equations has sparked a surge of interest in trajectory inference.
Early discrete methods, such as Waddington-OT (WOT)~\citep{schiebinger_reconstruction_2017}, used static OT to connect snapshots but lacked a continuous time model.
TrajectoryNet~\citep{tong_trajectorynet_2020} was the first to bridge this gap by regularizing a Continuous Normalizing Flow (CNF) with an OT-based kinetic energy penalty.
However, TrajectoryNet assumed strict mass conservation and was primarily tested on deterministic, non-spatial systems.

More recently, Flow Matching (FM) has emerged as a faster alternative to NDEs by regressing a vector field directly without full ODE integration during training~\citep{lipman2023flow, tong2023improving}.
Despite its efficiency, standard FM often assumes straight-line "conditional" probability paths in the ambient space.
This is a significant limitation for biological data, where trajectories must follow the high-curvature paths of a low-dimensional manifold.
Furthermore, while Riemannian Flow Matching~\citep{chen2023flow} attempts to address geometry, it requires the manifold to be analytically defined, which is not feasible for the complex, empirical graphs generated from single-cell data.

MIOFlow~\citep{huguet2022manifold} addressed this limitation by constructing a data-driven graph geometry via diffusion operators and constraining trajectories to the empirical manifold, ensuring they remain on the data manifold rather than cutting through ambient space. However, MIOFlow does not account for key biological factors such as cell proliferation or stochastic dynamics.

Importantly, none of these methods—including FM, WOT, or the original MIOFlow—explicitly account for the spatial context of cellular dynamics. They treat cells as isolated points in gene expression space, ignoring the tissue-scale signaling and neighborhood interactions that guide differentiation.

A separate line of work has begun to incorporate spatial information into trajectory inference. Methods such as SpaTrack~\cite{SHEN2025101194} and NicheFlow~\cite{sakalyan2025modeling} explicitly integrate spatial coordinates into models of cellular dynamics. However, these approaches differ from MIOFlow 2.0 in important respects: they rely on explicit spatial coordinates rather than derived spatial features, and, in the case of NicheFlow, operate at the level of cellular niches rather than individual cells (see Appendix~\ref{app:comparison} for a qualitative discussion). Moreover, neither method combines manifold-constrained transport with spatial conditioning in a unified framework.

MIOFlow 2.0 bridges these two lines of work, integrating manifold-constrained OT, stochastic dynamics, and spatial conditioning into a single unified model.

\section{Methods}
We propose MIOFlow 2.0, a unified framework for learning biologically plausible cellular trajectories. Our method operates in a latent space of a manifold-regularized autoencoder and learns a continuous time evolution $X_{u,t}$ that transports samples from an initial distribution $\mu_{t}$ to a target $\mu_{t+1}$. The framework consists of four components: (1) Manifold-Regularized Autoencoder to encode single-cell transcriptomics and spatial information (2) dynamic optimal transport, (3) The derivative network with auxiliary Population dynamics via learned growth/death rates, and (4) ODE/SDE integration.

Compared to the original MIOFlow, MIOFlow 2.0 has added the growth rate models, stochastic differential equations, and spatial information.

\subsection{The manifold-regularized autoencoder}
\label{sec:gaga}
Single-cell transcriptomic data and spatial features reside in a high-dimensional ambient space $\mathcal{X} \subset \mathbb{R}^D$ (where $D$ is the number of genes), yet biological constraints restrict valid cell states to a lower-dimensional manifold $\mathcal{M} \subset \mathcal{X}$. Standard trajectory inference methods often operate in the ambient space or use linear projections (e.g., PCA), which distort non-linear biological relationships. 
To respect the data's intrinsic geometry, we first map the raw gene expression $x \in \mathcal{X}$ to a lower-dimensional latent space $z \in \mathcal{Z} \subset \mathbb{R}^d$ using a geometry-aware autoencoder.
We employ the \textit{GAGA} framework~\citep{sun2024geometry}, which is specifically designed to preserve the diffusion geometry of the data. 
Unlike standard autoencoders that minimize only reconstruction error, our encoder $E_\phi: \mathcal{X} \to \mathcal{Z}$ is trained with a geometric loss that enforces isometry between the latent Euclidean distances and the manifold geodesic distances. 
We approximate these geodesic distances using PHATE~\citep{moon_visualizing_2019}, an information-theoretic distance metric that captures both local neighborhood structures and global non-linear trajectories via potential distances on a diffusion operator.
Formally, given a pair of cells $x_i, x_j$, we minimize the difference between their Euclidean distance in the latent space $\|E_\phi(x_i) - E_\phi(x_j)\|_2$ and their PHATE distance $D_{\text{ PHATE}}(x_i, x_j)$. 
By performing trajectory inference in this biologically informative latent space $\mathcal{Z}$, MIOFlow 2.0 ensures that the inferred transport costs correspond to true developmental distances rather than artifacts of high-dimensional noise.

\subsection{Inferring Trajectories with Dynamic OT}
Standard optimal transport finds the cheapest way to map cells from a starting distribution to a target distribution.
Dynamic optimal transport extends this by finding the continuous path that moves the mass over time while minimizing the total kinetic energy.
We model this continuous evolution of cell states in the latent space $\mathcal{Z}$.
Given a sequence of $T$ observed latent distributions $\{\mu_i\}_{i=0}^{T-1}$ corresponding to fixed timepoints $t \in \{0, \dots, T-1\}$, our goal is to learn a time-varying vector field $f_\theta(z,t)$ parametrized by a neural network.
This vector field defines the instantaneous rate of change for a cell $u$ in the latent space, represented by the differential equation:
$$ \frac{dZ_{u,t}}{dt} = f_\theta(Z_{u,t}, t). $$
To find the latent state of cell $u$ at any time $t$, we integrate this vector field forward in time using a Neural ODE:
\begin{equation}\label{eq: traj}
Z_{u,t} = Z_{u,0} + \int_0^t f_\theta(Z_{u,\tau},\tau)d\tau.
\end{equation}
This integration is subject to the condition that the transported latent cells match the observed latent data distributions, meaning $Z_{u,i} \sim \mu_i$ for all $i$.

\subsubsection{Theoretical Foundations}
Solving classic dynamic optimal transport is computationally difficult because it requires solving complex equations over entire distributions.
We adapt a theorem from \citet{tong_trajectorynet_2020} to demonstrate that we can solve this problem efficiently using a neural network.
This theorem proves that we can replace strict matching constraints with a soft penalty in our training loss.

\begin{thm}
\label{thm:dot}
Consider a time-varying vector field $f(z,t)$ defining latent cellular trajectories $dZ_{u,t} = f(Z_{u,t},t)dt$ with instantaneous density $\rho_t$, and a dissimilarity metric $D(\mu,\nu)$ such that $D(\mu,\nu)=0$ iff $\mu=\nu$.
Given these assumptions, there exists a sufficiently large regularization parameter $\lambda > 0$ such that the optimal transport problem satisfies:
\begin{equation}\label{eq: W2_in_thm}
W_2(\mu,\nu)^2 = \inf_{Z_{u,t}} \mathbb{E}\bigg[ \int_0^1 \|f(Z_{u,t},t)\|_2^2 dt \bigg] + \lambda D(\rho_1,\nu), \quad \text{s.t. } Z_{u,0} \sim \mu.    
\end{equation}

Moreover, because the process $Z_{u,t}$ is defined on the embedded manifold space $\mathcal{Z}$ learned by our geometry-aware autoencoder, the Euclidean Wasserstein distance in latent space approximates the geodesic Wasserstein distance on the ambient manifold: $W_2(\mu,\nu) \simeq W_{d_{\mathcal{M}}}(\mu,\nu)$.
\end{thm}

For proof, see \Cref{app:proof}.
This theorem is the theoretical justification for using a neural network to perform dynamic optimal transport.
The first term in Equation \ref{eq: W2_in_thm} minimizes the energy of the latent trajectories generated by the network.
The second term acts as a soft penalty to ensure the final predicted latent cells match the target data distribution.
By optimizing this relaxed function, the neural network learns the optimal continuous paths between latent cell states.
Because we operate directly in the latent space $\mathcal{Z}$, these inferred trajectories naturally follow the geometry of the data manifold.

\subsubsection{Training Procedure}
In practice, we observe discrete empirical distributions $\hat{\mu}_i := (1/n_i)\sum_{k=1}^{n_i} \delta_{z_k}$ for samples $z_k \in \mathcal{Z}_i$.
We approximate the integral in \eqref{eq: traj} using a Neural ODE ~\citep{chen2018neural}, which is a neural network denoted by $\psi_\theta: \mathbb{R}^d \times \mathcal{T} \to \mathbb{R}^{d|\mathcal{T}|}$.
The derivatives output by the network are then integrated via an ODE solver to forecast cellular trajectory values over time for each cell.
Given an initial set of points $\mathcal{Z}_0$, the solver predicts their future states $\hat{\mathcal{Z}}_1, \dots, \hat{\mathcal{Z}}_{T-1} = \psi_\theta(\mathcal{Z}_0, \{1, \dots, T-1\})$.
We employ two training strategies to enforce the marginal constraints.
\begin{enumerate}
\item Local Training: We integrate sequentially between adjacent timepoints.
Given observed data $\mathcal{Z}_t$ at time $t$, we predict $\hat{\mathcal{Z}}_{t+1} = \psi_\theta(\mathcal{Z}_t, t+1)$.
This focuses on matching the immediate transition.
\item Global Training: We integrate the entire trajectory starting from the initial distribution $\mathcal{Z}_0$ to predict all subsequent timepoints.
\end{enumerate}

The overall training objective combines three loss terms described as follows.
\begin{equation}
L = \lambda_m L_m + \lambda_e L_e + \lambda_d L_d.
\end{equation}
1. Marginal Matching Loss ($L_m$): This term corresponds to the relaxation in \eqref{eq: W2_in_thm}, enforcing that transported cells match the observed population at the target timepoint.
We use the Wasserstein-2 distance between the predicted distribution $\hat{\mu}_i$ (from $\hat{\mathcal{Z}}_i$) and the observed empirical distribution $\mu_i$.
\begin{equation}
\label{eq: loss_m} L_m := \sum_{i=1}^{T-1} W_2(\hat{\mu}_i, \mu_i).
\end{equation}
We compute $W_2$ using the POT library~\citep{flamary2021pot}.
This offers a significant computational advantage over Continuous Normalizing Flows (CNF), which rely on maximum likelihood estimation and require computing the trace of the Jacobian at a cost of $O(d^2)$ per function evaluation~\citep{chen2018neural}.

2. Transport Energy Regularization ($L_e$): This term minimizes the kinetic energy of the flow, ensuring that cells follow the most direct paths (straight lines in the latent space $\mathcal{Z}$, which correspond to geodesics on the manifold).
This loss facilitates the dynamic optimal transport solution.
We approximate the integral using the solver evaluations.
\begin{equation}
\label{eq: loss_e} L_e := \sum_{i=1}^{T-1} \int_{i-1}^{i} \|f_\theta(Z_{u,t},t)\|_2^2 dt.
\end{equation}

3. Manifold Density Loss ($L_d$): Inspired by \citet{tong_trajectorynet_2020}, we add a density regularization to encourage cellular trajectories to stay within high-density regions of the manifold, preventing shortcuts through empty space.
For a predicted point $z \in \hat{\mathcal{Z}}_t$, we penalize deviations from the $k$-nearest neighbors in the observed data $\mathcal{Z}_t$.
\begin{equation}
L_d:= \lambda_d\sum_{t=1}^{T-1}\sum_{z\in\hat{\mathcal{Z}}_t}\ell_d(z,t), \text{ where } \ell_d(z,t) := \sum_{i=1}^k \max(0,\text{min-k}(\{\|z-y\|:y\in\mathcal{Z}_t\})-h). \label{eq: loss_d}
\end{equation}
where $h > 0$ is a margin parameter.

\subsection{Modeling Stochastic Dynamics with Neural SDEs}\label{sec:sde}
While Neural ODEs effectively model deterministic trends, cellular trajectories are inherently stochastic.
They are often driven by noisy gene expression and probabilistic fate decisions~\citep{elowitz2002stochastic, losick2008stochasticity}.
In such systems, a cellular state may stochastically branch into multiple distinct lineages.
Deterministic flows cannot naturally represent this phenomenon because they tend to draw straight paths across empty state space.
As illustrated in Figure \ref{fig:ablation}, a deterministic baseline struggles to capture the complex geometry of a branching dataset.
To capture this intrinsic variability and follow the diverging manifold structure, we extend our framework from deterministic ODEs to Neural Stochastic Differential Equations.

\subsubsection{Neural Stochastic Differential Equations Formulation}
The mathematical formulation of Neural SDEs treats the deterministic ODE as a special case where the diffusion term is zero~\citep{chen2018neural, Protter2005Stochastic, Tzen2019NeuralSDE}.
We represent the latent state evolution as a $d$-dimensional diffusion process $Z = \{Z_t\}_{t \in [0,T]}$ given by the solution of the Ito SDE.
\begin{equation}
\label{eq:itosde} dZ_t = b(Z_t, t; \theta)dt + \sigma(Z_t, t; \phi)dW_t,
\end{equation}
where $W_t$ is a standard $d$-dimensional Wiener process.
Biologically, the drift term $b$ models the deterministic developmental program driving the cell forward.
The diffusion term $\sigma$ models the transcriptional noise and uncertainty at developmental branching points.
In this framework, the drift $b$ and the diffusion coefficient $\sigma$ are parameterized as neural networks with weights $\theta$ and $\phi$, respectively~\citep{Tzen2019NeuralSDE}.
As established by \citet{Tzen2019Theoretical}, for a target density $q(z) = f(z)\phi_d(z)$, this formulation can obtain an information-optimal exact sample of the target distribution at $t=1$.
Furthermore, when the networks $b$ and $\sigma$ are Lipschitz-continuous in $z$ uniformly in $t$, there exists a progressively measurable mapping between the Wiener process and the latent state $Z$~\citep{Bichteler2002, Tzen2019NeuralSDE}.
This ensures the existence and uniqueness of the learned cellular trajectories.

\subsubsection{Momentum-Based Trajectory Refinement}
To account for the global dynamics of the biological process and preserve the manifold structure in sparse regions, we guide the trajectories with a momentum term.
We refine the drift term $b(Z_t, t)$ by incorporating historical velocity:
\begin{equation}
\label{eq:momentum_drift}
dZ_t = \left[ (1- \beta)f_\theta(Z_t,t) + \beta v(Z_t,t) \right] dt + \sigma(Z_t, t; \phi)dW_t,
\end{equation}
where $v(Z_t, t)$ is the momentum term representing an average of the previous drift:
\begin{equation}
\label{eq:momentum_v}
v(Z_t,t) := \int_{0}^{t} w(\xi) f_\theta(Z_{t-\xi}, t-\xi) d\xi,
\end{equation}
and $w(\cdot): \mathbb{R}^+ \to [0,1]$ is a weight function.
This momentum term acts as a stabilizer, smoothing the gradients and ensuring that the inferred trajectories maintain a consistent direction that respects the global structure of the data manifold. In other words, instead of taking sharp turns to match existing populations, momentum creates more gradual curved trajectories that account for missing populations. 


\subsection{Modeling Cell Proliferation and Death}
\label{sec:growth}

Biological populations are non-conservative, as processes such as cell proliferation and apoptosis continuously reshape the distribution of cell states over time.
Classical optimal transport enforces a strict conservation of mass, which can lead to biased trajectories when trying to match snapshots with non-uniform growth or death.
To account for these non-conservative dynamics, we introduce a proliferation neural network $h_\psi(z, t)$ that estimates the local growth or death rate for a cell in the latent state $z$ at time $t$.
For a cell at state $z$, $h_\psi(z, t)$ estimates the expected relative mass or number of descendants at a subsequent timepoint.
Values of $h_\psi(z, t) < 1$ indicate cell death or exit from the population, while values $h_\psi(z, t) > 1$ indicate proliferation.
This modifies the predicted marginal distribution $\hat{\mu}$ by weighting the transported samples:
\begin{equation}
\label{eq:weighted_mu}
\hat{\mu}(\cdot) = \frac{\sum_{j} h_\psi(z_j, t) \delta_{z_j}(\cdot)}{\sum_{j} g_\psi(z_j, t)}.
\end{equation}
The resulting weighted distribution is utilized in the marginal matching loss $L_m$:
\begin{equation}
\label{eq:Lm_growth}
L_m := \sum_{i=1}^{T-1} W_2 (\hat{\mu}_i, \mu_i),
\end{equation}
where $\mu_i$ is the ground-truth discrete distribution at time $i$ with uniform weights.
By incorporating $h_\psi$, MIOFlow 2.0 can effectively model systems with drastic population shifts, such as those found in cancer treatment response or rapid embryonic expansion.

\subsubsection{Initialization via Unbalanced Optimal Transport}
To ensure the proliferation network is biologically grounded, we initialize its weights using the marginals of static unbalanced optimal transport (UOT).
UOT relaxes the marginal matching constraints, allowing for mass creation and destruction between distributions.
For each pair of adjacent timepoints $t$ and $t+1$, we solve for the optimal non-negative coupling $\pi^*$:
\begin{equation}
\label{eq:uot_coupling}
\pi^* := \arg\min_{\pi \ge 0} \left\{ \int \|z_t - z_{t+1}\|^2\,d\pi(z_t, z_{t+1}) + \lambda_1 D_{\mathrm{KL}}(\pi_1 \| \mu_t) + \lambda_2 D_{\mathrm{KL}}(\pi_2 \| \mu_{t+1}) \right\},
\end{equation}
where $\pi_1$ and $\pi_2$ are the first and second marginals of the transport plan $\pi$.
We choose a small relaxation parameter $\lambda_1$ and a large $\lambda_2$ to allow for significant mass changes at time $t$ while encouraging strong alignment with the target distribution at time $t+1$.

The initial growth function is then estimated using the source marginal of the optimal coupling:
\begin{equation}
\label{eq:growth_init}
g_\psi(z_t, t) \approx \pi^*_1(z_t) := \int \pi^*(z_t, z_{t+1}) \, \text{d}z_{t+1}.
\end{equation}
This initialization provides a "warm-start" for the proliferation network, which is then refined during the joint training of the continuous vector field and the growth rate model.


\begin{figure}
    \centering
    \includegraphics[width=1.0\linewidth]{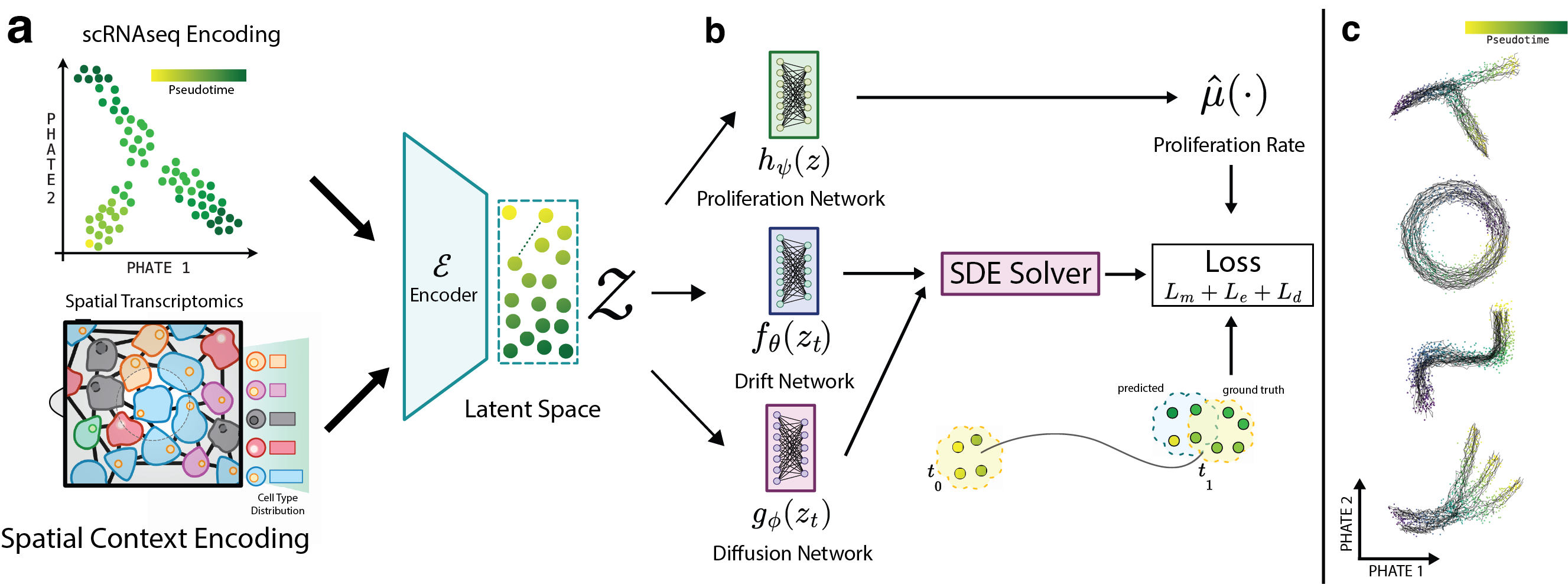}
    \caption{\textbf{Overview of the MIOFlow model.} \textbf{A.} We initialize with scRNA-seq data and spatial transcriptomics, then concatenate both feature sets into a jointly embedded latent space. Each data point in this latent space represents a cell embedding informed by its neighbors. \textbf{B.} The embedding serves as input to three networks: a proliferation network predicting the proliferation rate, and drift and diffusion networks comprising the SDE/ODE model. \textbf{C.} The resulting trajectories can be visualized over the latent space.}
    \label{fig:cond-network}
\end{figure}

\subsection{Spatially-Aware Dynamics via Joint Embeddings}
\label{sec:spatial}
A key innovation of MIOFlow 2.0 is the integration of cellular microenvironment and tissue organization into the definition of the cell state.
We posit that the vector field governing cell state evolution depends not only on the intrinsic transcriptional state of a cell but also on its local microenvironment and the signals it receives from its neighbors.
This dependency is a fundamental principle across biological systems: stem cell niches direct differentiation fate~\citep{scadden2006stem, lander2011pattern}, while in oncology, the tumor microenvironment drives cancer cell plasticity, metastasis, and therapy resistance~\citep{quail2013microenvironmental, hanahan2011hallmarks}.
Similarly, immune cell states are heavily modulated by tissue context, where local signaling networks can induce polarization or exhaustion distinct from lineage-intrinsic programs~\citep{gajewski2013cancer, binnewies2018understanding}.
To capture these dual drivers of dynamics, we construct a joint feature space that unifies the intrinsic gene expression profile with the extrinsic spatial context.


\subsubsection{Joint Feature Construction and Embedding}
We define the state of cell $i$ as a pair of vectors $(x_g^{(i)}, x_s^{(i)})$.
Here, $x_g^{(i)} \in \mathbb{R}^{D_g}$ represents the intrinsic gene expression profile.
The vector $x_s^{(i)} \in \mathbb{R}^{D_s}$ encodes the spatial context and is computed directly by aggregating features from the cell's local neighbors, $\mathcal{N}(i)$.
The spatial vector $x_s^{(i)}$ is computed from three features:
1. Neighborhood Composition: A vector containing the frequency of each cell type within $\mathcal{N}(i)$.
2. Interaction Potential: For known ligand-receptor pairs, the product of the receptor expression in cell $i$ and the sum of  ligand expression in $\mathcal{N}(i)$.
3. Local Expression Niche: The mean PCA embedding vector of $\mathcal{N}(i)$. These spatial modalities are concatenated and dimensionality-reduced by PCA to obtain a single spatial feature vector, $x_s^{(i)}$.

We then construct a low-dimensional joint gene-spatial embedding by projecting each of $x_s^{(i)}$ and $x_g^{(i)}$ through their own Geometry-aware autoencoders $E_\phi$ and $E_\theta$ respectively, yielding $z^{(i)} = \left[E_\phi(x_{s}^{(i)}),\, s \cdot E_\theta(x_g^{(i)})\right]$. Here, $s$ is a hyperparameter which controls the contribution of the spatial embedding to the overall trajectory.
The Neural SDE is trained to learn the trajectory $Z_t$ directly on this joint manifold:
\begin{equation}
dZ_t = f_\theta(Z_t, t) dt + \sigma(Z_t, t) dW_t.
\end{equation}
This approach ensures that the inferred dynamics are driven by the complete cellular state, allowing the model to distinguish between cells with identical expression but distinct environmental histories.

Below, we explain in detail the construction of the spatial features.

\subsubsection{Spatial Feature Extraction}
We extract spatial features through a structured graph-based methodology that summarizes the neighborhood of each cell represented on Figure \ref{fig:spatial_feature}.

\begin{figure}
    \centering
    \includegraphics[width=1.0\linewidth]{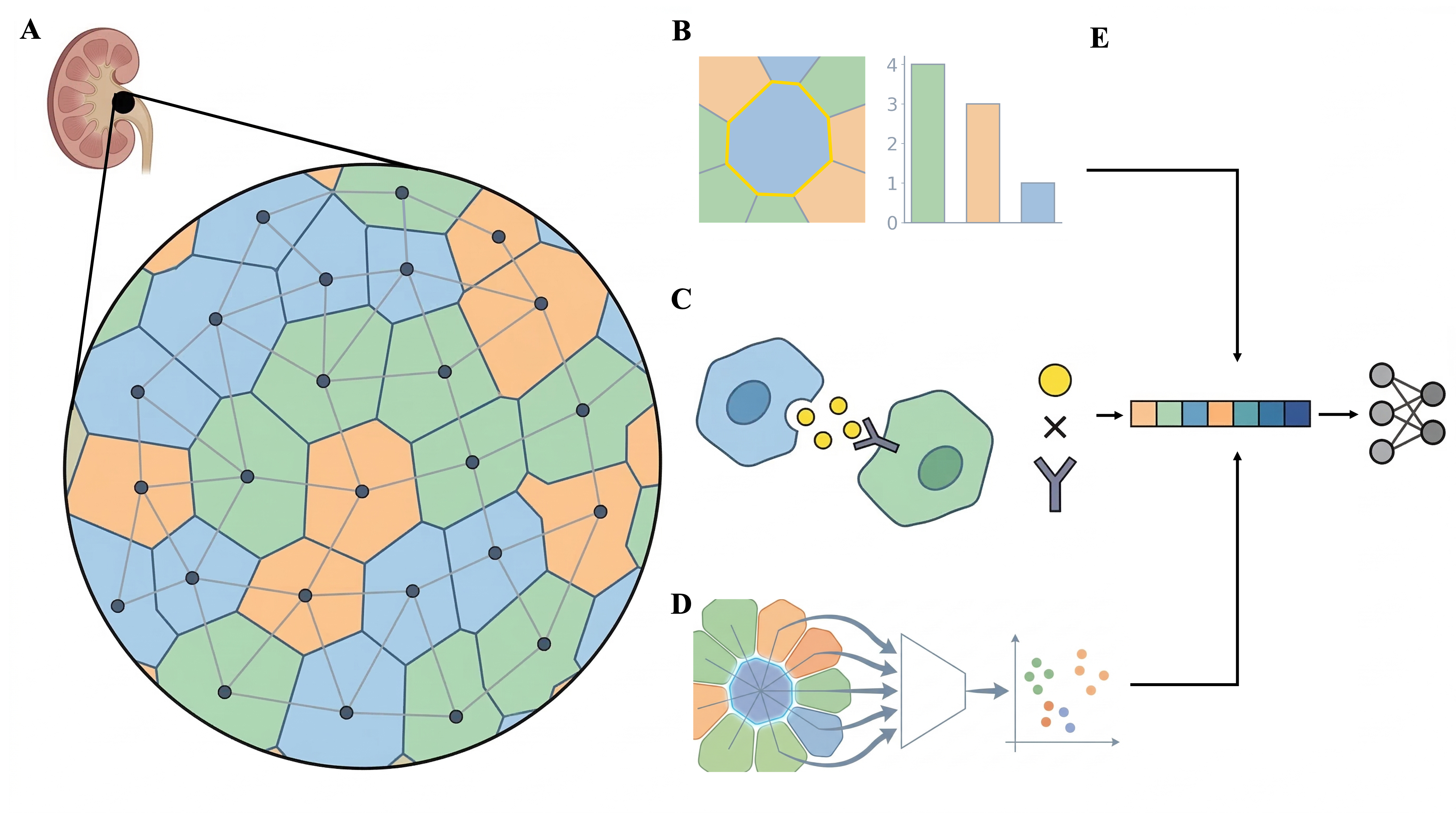}
    \caption{\textbf{ features Extraction for MIOFlow 2.0.} \textbf{A.} Build the neighborhood graph using knn graph or Voronoi polygons \textbf{B.} Compute local cell type frequency from neighborhood. Colors indicate cell types.  \textbf{C.} Compute ligand-receptors signalling strength from neighbors to target cell. \textbf{D.} Local Expression Niche: The mean PCA embedding vector of neighboring cells  \textbf{E.} Concatenate cell features with their spatial neighbors information}
    \label{fig:spatial_feature}
\end{figure}

First, we construct a cell-cell graph based on the spatial coordinates provided by spatial transcriptomics data, typically using k-nearest neighbors or Voronoi polygons, which is extended to include edges between all points within some distance $d$ on the initial graph. We then apply message passing on this graph to obtain aggregated node features that describe the neighborhoods of each cell, which we refer to as spatial features.

\subsubsection{Neighborhood Composition and Signaling}
We characterize the local environment through several distinct spatial descriptors.
For categorical features such as cell types, we convert the categories into one-hot vector encodings and sum over all neighbors to compute local cell type frequencies.
To explicitly model cell-cell communication, we calculate ligand-receptor interaction potentials for known signaling pairs~\cite{jin_cellchat}.
For each such pair, we sum the product of the receptor expression in the target cell by the ligand expression of each of its spatial neighbors. This aggregate feature quantifies the "input signal" that cell experiences from local cell signalling.

\subsubsection{Local Expression Niche}
We compute a local expression niche feature to capture the broader transcriptional state of the microenvironment.
We first project the raw gene expression data into a lower dimensional space using principal component analysis.
We then calculate the mean of these principal component embeddings across the spatial neighbors of each cell.
This step yields a concise summary of the gene expression profile surrounding the target cell.
To form the final spatial feature representation, we concatenate the neighborhood composition, signaling potentials, and the local expression niche.
We apply a secondary principal component analysis to this concatenated vector to reduce its dimensionality.

As detailed in Algorithm \ref{alg:spatial_features}, these enriched spatial features are used to condition the neural differential equation.
This allows the learned vector field to vary with the microenvironment.

\begin{algorithm}[h]
\caption{Computing Spatial Features}
\label{alg:spatial_features}
\begin{algorithmic}[1]
\State Input: 
\begin{itemize}
    \item Cell-by-gene expression matrix $X = (x_{cg})_{c\in\mathcal{C},\,g\in\mathcal{G}}$
    \item Cell-by-location matrix $Y = ((y_{c1},y_{c2}))_{c\in\mathcal{C}}$
    \item Cell type assignments $\tau = (\tau_c)_{c\in\mathcal{C}}$ with $\tau_c\in\{1,\dotsc,M\}$
    \item Ligand-Receptor pairs $\mathcal{P} = \{(l,r) : l,r\in\mathcal{G}, \, l \text{ is a ligand}, \, r \text{ is a receptor, and } l\to r \text{ is a known pair}\}$
\end{itemize}
\State Output: Spatial features $S$
\State $G \gets \texttt{MakeGraph}(Y)$ \Comment{Construct cell-cell graph from locations}
\State $E \gets \texttt{PCA}(X)$ \Comment{Compute gene expression embeddings}
\For{each cell $c\in \mathcal{C}$}
    \State $\mathcal{N}(c) \gets \texttt{Neighbors}(G,c)$
    \State $m_{c} \gets \frac{1}{|\mathcal{N}(c)|} \sum_{c' \in \mathcal{N}(c)} E_{c'}$ \Comment{Local expression niche}
    \State $h_{c} \gets \texttt{CellTypeFrequencies}(\mathcal{N}(c))$ \Comment{$h_{c}\in \mathbb{R}^{M}$: frequency vector}
    \State $a_{c} \gets 0 \in \mathbb{R}^{|\mathcal{P}|}$ \Comment{Initialize ligand-receptor vector}
    \For{each pair $(l,r) \in \mathcal{P}$ with index $k$}
        \For{$c'\in \mathcal{N}(c)$}
            \State $a_{c}[k] \gets a_{c}[k] + x_{c'l} \cdot x_{c r}$ \Comment{Accumulate interaction potential}
        \EndFor
        \State $a_{c}[k] \gets \frac{a_{c}[k]}{|\mathcal{N}(c)|}$
    \EndFor
\EndFor
\State $M_{pca} \gets (m_{c})_{c\in \mathcal{C}}$
\State $H \gets (h_{c})_{c\in \mathcal{C}}$
\State $A \gets (a_{c})_{c\in \mathcal{C}}$
\State $S_{raw} \gets \texttt{NormalizeAndConcatenate}(H,\,M_{pca},\,A)$
\State $S \gets \texttt{PCA}(S_{raw})$ \Comment{Dimensionality reduction of joint spatial features}
\State \Return $S$
\end{algorithmic}
\end{algorithm}

\subsection{MIOFlow 2.0 Framework}
\label{sec:full_framework}
We integrate the manifold-constrained optimal transport, stochastic differential dynamics, proliferation modeling, and spatial conditioning into a unified framework.
This integration allows MIOFlow 2.0 to learn trajectories that are both geometrically faithful to the data manifold and biologically plausible.
An overview of the complete learning procedure for MIOFlow 2.0 is detailed in Algorithm \ref{alg:mioflow_simplified}. A full detailed description is provided on Appendix \ref{app:full_mioflow}.

As shown in the algorithm, the framework is flexible and can operate in different modes depending on the availability of spatial data or the specific solver required for the biological system.

\subsubsection{Unified Training Objective}
The final objective function is a combination of the transport, manifold, and biological priors described in the previous sections.
During each training iteration, we sample batches of cells and their corresponding spatial conditions across timepoints.
The model parameters for the drift $b_\theta$, diffusion $\sigma_\phi$, and growth rate $g_\psi$ are updated simultaneously to minimize the total loss.
The use of the PHATE-distance matching autoencoder ensures that all differential equations are solved within a latent space where Euclidean distances approximate biological geodesic distances.

\subsubsection{Trajectory Inference and Prediction}
Once trained, MIOFlow 2.0 can be used to predict the continuous-time evolution of any cell state.
By providing an initial transcriptional state and an optional sequence of spatial environments, the model integrates the learned SDE to generate a distribution of likely future states.
This capability allows for the discovery of critical branching points and the identification of environmental drivers that steer cells toward specific fates.
The spatial conditioning mechanism, in particular, enables the simulation of "what-if" scenarios, such as predicting how a regenerative trajectory might change if the surrounding cell type composition were altered.

\begin{algorithm}[H]
\caption{MIOFlow 2.0}
\label{alg:mioflow_simplified}
\begin{algorithmic}[1]
\State {\bfseries Input:} Time-series gene expression data $\{X_1,\dots,X_T\}$, optional spatial features $\{S_1,\dots,S_T\}$
\State {\bfseries Output:} Trained trajectory model parameters $\theta,\phi,\psi$
\State
\State $\{Z_1,\dots,Z_T\} \gets \texttt{GAGA}(\{X_1,\dots,X_T,S_1,\dots,S_T\})$ \Comment{Co-Embed in low-dimensional space}

\State
\For{each training iteration}
  \State Sample mini-batches $\{\mathbf{z}_t\}_{t=1}^T$ and $\{\mathbf{s}_t\}_{t=1}^T$ (if spatial features provided)
  \If{local mode}
    \State Integrate from $t$ to $t+1$ using dynamics $f_\theta, g_\phi$ (informed by $\mathbf{s}_t$ if available)
  \Else \Comment{global mode}
    \State Integrate from the first timepoint $t=1$ to the entire trajectory (informed by $\mathbf{s}_t$if available)
  \EndIf
  \State Predict cell masses using learned growth model $h_\psi$
  \State $Loss \gets L_{m}, L_{e}, L_{d}$
  \State Update dynamics parameters $\theta,\phi$ and growth model parameters $\psi$ via gradient descent
\EndFor
\State \Return trained parameters
\end{algorithmic}
\end{algorithm}

\section{Results}

In order to validate MIOFlow 2.0 we used problems of increasing difficulty. We started using the SERGIO simulator to obtain a synthetic but biologically relevant baseline. Then we did an ablation study on sythetic data to show the effectiveness of our biologically plausible designs. Lastly we conducted a case study running our algorithm in a real-world axolotl spatial transcriptomics \cite{wei_axolotl}.

\subsection{Synthetic single-cell data generation with SERGIO}

We use SERGIO \cite{SERGIO}, a stochastic gene regulatory network (GRN) simulator, to generate synthetic single-cell gene expression data with known ground-truth regulatory structure and temporal dynamics.
SERGIO simulates gene expression by modeling the interactions between master regulators and target genes through a system of stochastic differential equations.
By incorporating both intrinsic transcriptional variability and technical noise—such as library size effects and dropout—the simulator produces sparse count matrices that emulate the challenges of experimental scRNA-seq data.
This approach allows us to benchmark MIOFlow 2.0 on populations with captured progenitor states, intermediate transitions, and terminal fates where the underlying biological "truth" is explicitly defined. We generated two synthetic datasets representing trajectory topologies commonly observed in single-cell studies: a trifurcating differentiation process and a curved, S-shaped trajectory. The full mathematical formulation of the regulatory dynamics and noise models is detailed in Appendix \ref{app:sergio}.

\subsubsection{Trifurcation trajectories}

The trifurcation dataset simulates a differentiation process in which a single progenitor cell state gives rise to three distinct terminal cell fates. We simulated 100 genes across 500 cells, organized according to a differentiation graph with one initial cell type and three downstream cell types.  This was achieved by defining a differentiation graph in which one initial cell type has outgoing transitions to three downstream cell types, each associated with a distinct master regulator expression program. Cells were sampled along all three trajectories, producing a continuous branching structure with a shared progenitor region and three diverging lineages.

\subsubsection{S-shaped trajectory}

The S-shaped dataset models a more complex trajectory involving cell-cycle dynamics coupled with fate specification. We simulated a cell-cycle driven differentiation process consisting of 1000 genes and 990 cells. In this setting, cells first undergo a cyclic progression corresponding to cell-cycle dynamics. From a specific point along the cycle, the trajectory bifurcates into two terminal fates. The final S-shaped dataset was obtained by subsampling 315 cells along the combined cycle and bifurcation trajectories, resulting in a curved differentiation path.

\subsubsection{Synthetic Results}
\label{sec:results_synthetic}
We compared MIOFlow 2.0 against several baseline trajectory inference methods on both the trifurcation and S-shaped SERGIO datasets.
The baseline methods includes Schrodinger Bridges \cite{pavon2021data,de2021diffusion} (using an ODE and SDE solvers) and Conditional Flow Matching \cite{grathwohl2018ffjord, chen2018neural}.
These methods represent diverse approaches to modeling cellular trajectories, including stochastic bridge processes and continuous normalizing flows implemented via neural ODEs.

In order to evaluate each method we follow the same procedure.
First, we obtained the latent space of the simulation using PHATE \cite{moon_visualizing_2019}.
Then in order to evaluate each method we used a hold-out strategy at each timepoint.

Let $\mathcal{D} = \{(z_i, t_i)\}_{i=1}^N$ denote the dataset where $z_i \in \mathbb{R}^d$ represents the latent coordinates and $t_i \in \{0, 1, \dots, T-1\}$ denotes the discrete timepoint.

For each timepoint $t$, we randomly partition the samples into a training set $\mathcal{D}_{train}^{(t)}$ and a test set $\mathcal{D}_{test}^{(t)}$, withholding a fraction $\alpha$ of samples for testing.
This stratified splitting ensures that test samples are available at every timepoint to assess reconstruction quality across the full temporal range.

Each trajectory inference method is trained exclusively on $\mathcal{D}_{train}$ to learn a velocity field $v_\theta(z, t)$ that captures the underlying cellular trajectories.
The learned vector fields can be integrated as either an ordinary differential equation (ODE) or, for methods that model stochasticity, a stochastic differential equation (SDE).

\paragraph{Trajectory Generation.}
Starting from training samples at $t=0$, we generate trajectories by integrating the learned vector field over the time interval $[0, T-1]$.
\begin{equation}
\frac{dz}{dt} = v_\theta(z, t), \quad z(0) \sim \mathcal{D}_{train}^{(0)}
\end{equation}

The resulting trajectories are clustered into $K$ branches using $k$-means clustering on their endpoints, and a mean trajectory $\boldsymbol{\mu}_k(t)$ is computed for each branch $k \in \{1, \ldots, K\}$. $K$ is determined depending on the number of endpoint in the simulation, for the trifurcation $k=3$ and S-shaped $k=1$ (same as a simple average).

\paragraph{Evaluation Metric.}
For each test sample $\mathbf{x} \in \mathcal{D}_{\text{test}}$, we compute the minimum Euclidean distance to the nearest point on any mean branch trajectory:
\begin{equation}
    d(\mathbf{x}) = \min_{k \in \{1,\ldots,K\}} \min_{t \in [0, T-1]} \|\mathbf{x} - \boldsymbol{\mu}_k(t)\|_2
\end{equation}

We report the mean and standard deviation of $d(\mathbf{x})$ across all test samples as the trajectory reconstruction error. This metric quantifies how well the learned trajectories cover the held-out data distribution at each timepoint.

\begin{figure}[h!]
  \centering
  \includegraphics[width=0.95\textwidth]{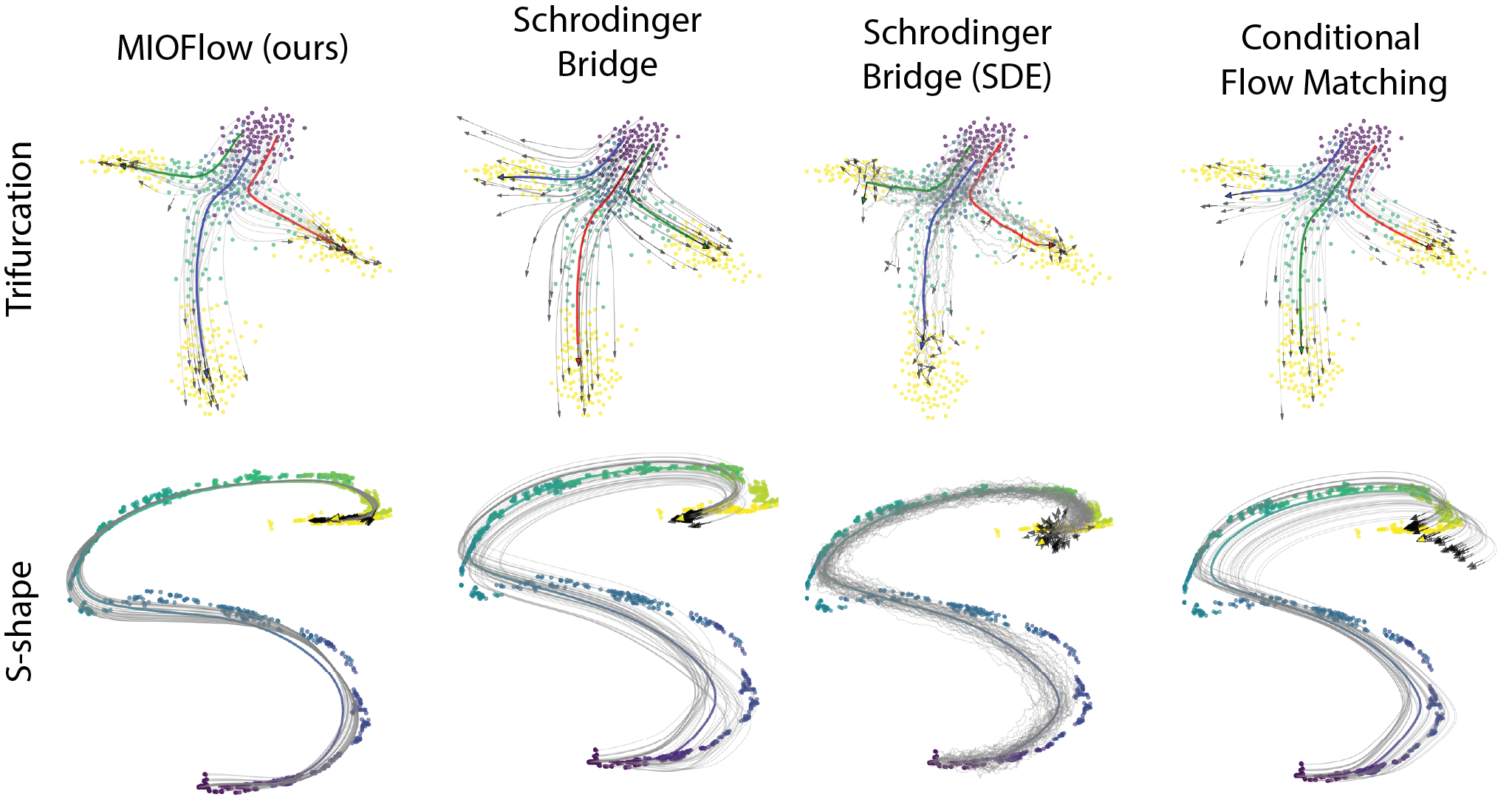}
  \caption{Comparison of trajectory inference methods on synthetic SERGIO datasets. (Top) Trifurcation dataset with three terminal fates. (Bottom) S-shaped dataset with cyclic progression and bifurcation. Cells are colored by ground truth timepoint. Predicted mean branch trajectories for each method are shown as colored curves. MIOFlow 2.0 maintains close adherence to the data manifold, while baseline methods deviate into unpopulated regions or fail to capture fine-scale curvature.}
  \label{fig:synthetic_comparison}
\end{figure}

The quantitative results for the SERGIO datasets are summarized in Table~\ref{tab:results}. MIOFlow 2.0 achieved the lowest reconstruction error on both topologies. Moreover, we can observe on Figure \ref{fig:synthetic_comparison} the results of each method on the PHATE \cite{moon_visualizing_2019} latent space. On the trifurcation datasets, the other methods adhere somewhat to the manifold, but their trajectories fail to capture curved details and tend to traverse regions where no points are sampled. The S-shaped dataset demonstrates how accounting for the latent space in a geometrical manner makes a difference when computing trajectories. While these methods can more or less follow the shape of the data, sampling from their trajectories would generate substantial measurement errors, since they largely fall outside the manifold. MIOFlow 2.0, by contrast, adheres faithfully to the data manifold. 

\begin{table}[h]
\centering
\caption{Trajectory reconstruction error (mean $\pm$ std) from held-out test samples to the nearest point on predicted trajectories.}
\label{tab:results}
\setlength{\tabcolsep}{8pt}
\begin{tabular}{lcc}
\toprule
 & \multicolumn{2}{c}{\textbf{Dataset}} \\
\cmidrule(lr){2-3}
\textbf{Method}      & S-shaped & Trifurcation \\
\midrule
MIOFlow              & $0.082 \pm 0.070$ & $0.243 \pm 0.202$ \\
SF2M (ODE)           & $0.134 \pm 0.122$ & $0.278 \pm 0.240 $ \\
SF2M (SDE)           & $0.115 \pm 0.091$ & $0.320 \pm 0.285$ \\
OT-CFM               & $0.118 \pm 0.095$ & $0.298 \pm 0.250$ \\
\bottomrule
\end{tabular}
\end{table}

\subsection{Results on Single Cell Embryoid Body Data}
\label{sec:results_ebd}
To evaluate trajectory inference accuracy on a biological data, where there is no ground truth for trajectories, we employ a leave-one-out (LOO) protocol on a real world Embryoid Body dataset \cite{moon_visualizing_2019} which comprises five time points. For each held-out time point $t$, each method is trained on the remaining time points and tasked with predicting the cell distribution at $t$ by integrating from the preceding time point $t_{-1}$. Table \ref{tab:loo_eb} reports the results for all interior time points (i.e., excluding the first and last). To ensure fair comparison, all methods are trained and evaluated in the same feature space: the GAGA 2D latent space ensuring that all methods operate on an identical representation. Predicted and ground-truth cell distributions are compared using three metrics: (1) the 1-Wasserstein distance (W1), computed via the Earth Mover's Distance with a Euclidean ground metric on subsampled populations of up to 1,000 cells; (2) the Maximum Mean Discrepancy with a Gaussian kernel (MMD-G), using median-heuristic bandwidth estimation; and (3) an $\ell_2$-norm discrepancy between sample means (MMD-M), which corresponds to the MMD under an identity-map kernel and measures distributional shift at the level of the mean embedding.

\begin{table}[t]
\caption{Leave-one-out trajectory inference accuracy on the EB dataset. Metrics are computed in raw PCA space (50D). W1: Wasserstein-1; MMD-G: Gaussian MMD; MMD-M: mean discrepancy. Best values are in bold.}
\label{tab:loo_eb}
\centering
\small
\begin{tabular}{l l c c c}
\hline
\textbf{Method} & \textbf{Held-out $t$} & \textbf{W1} $\downarrow$ & \textbf{MMD-G} $\downarrow$ & \textbf{MMD-M} $\downarrow$ \\
\hline

MIOFlow & 1 & 0.8072 & \textbf{0.1396} & \textbf{0.7700} \\
MIOFlow & 2 & \textbf{0.4689} & \textbf{0.0255} & \textbf{0.1973} \\
MIOFlow & 3 & \textbf{0.4101} & \textbf{0.0209} & \textbf{0.3646} \\
\hline

OT-CFM & 1 & \textbf{0.7924} & 0.1442 & 0.7709 \\
OT-CFM & 2 & 0.4741 & 0.0296 & 0.2034 \\
OT-CFM & 3 & 0.4612 & 0.0299 & 0.4343 \\
\hline

SF2M & 1 & 0.8694 & 0.1540 & 0.8428 \\
SF2M & 2 & 0.4864 & 0.0356 & 0.2501 \\
SF2M & 3 & 0.4667 & 0.0316 & 0.4448 \\
\hline

\end{tabular}
\end{table}

\subsection{Ablation Study}
To evaluate the individual effectiveness of the stochastic dynamics modeling and the growth/death rate model, we conducted an ablation study using three synthetic datasets. These datasets were specifically designed to isolate and simulate three fundamental biological behaviors: 
\begin{enumerate}
    \item branching dynamics, representing cell fate decisions; 
    \item population proliferation, simulating rapid growth;
    \item population attrition, simulating cell death.
\end{enumerate}
These scenarios represent common pitfalls for standard flow-based models, which typically assume a constant population density and deterministic paths.

We compared a baseline version of MIOFlow 2.0 consisting only of the manifold-constrained optimal transport against versions incrementally incorporating our biological priors. As shown in Figure~\ref{fig:ablation}, the baseline model is capable of interpolating between snapshots but introduces spurious trajectories between branches (left: there are trajectories between branches that do not overlap with the cells), across populations (middle: there are ), and fails to match the growing population (right).
In contrast, adding the growth rate model does not induce the trajectories that
Furthermode, the 

Furthermore, the deterministic baseline struggles to capture the complex geometry of the branching dataset. In contrast, the integration of Neural Stochastic Differential Equations allows the model to capture the probabilistic nature of the branching event, producing trajectories that more faithfully follow the diverging manifold structure. These results demonstrate that while manifold-constrained OT provides a strong geometric foundation, the explicit modeling of stochasticity and proliferation is essential for capturing the nuances of realistic biological processes.
\begin{figure}[h]
  \centering
  \includegraphics[width=0.5\textwidth]{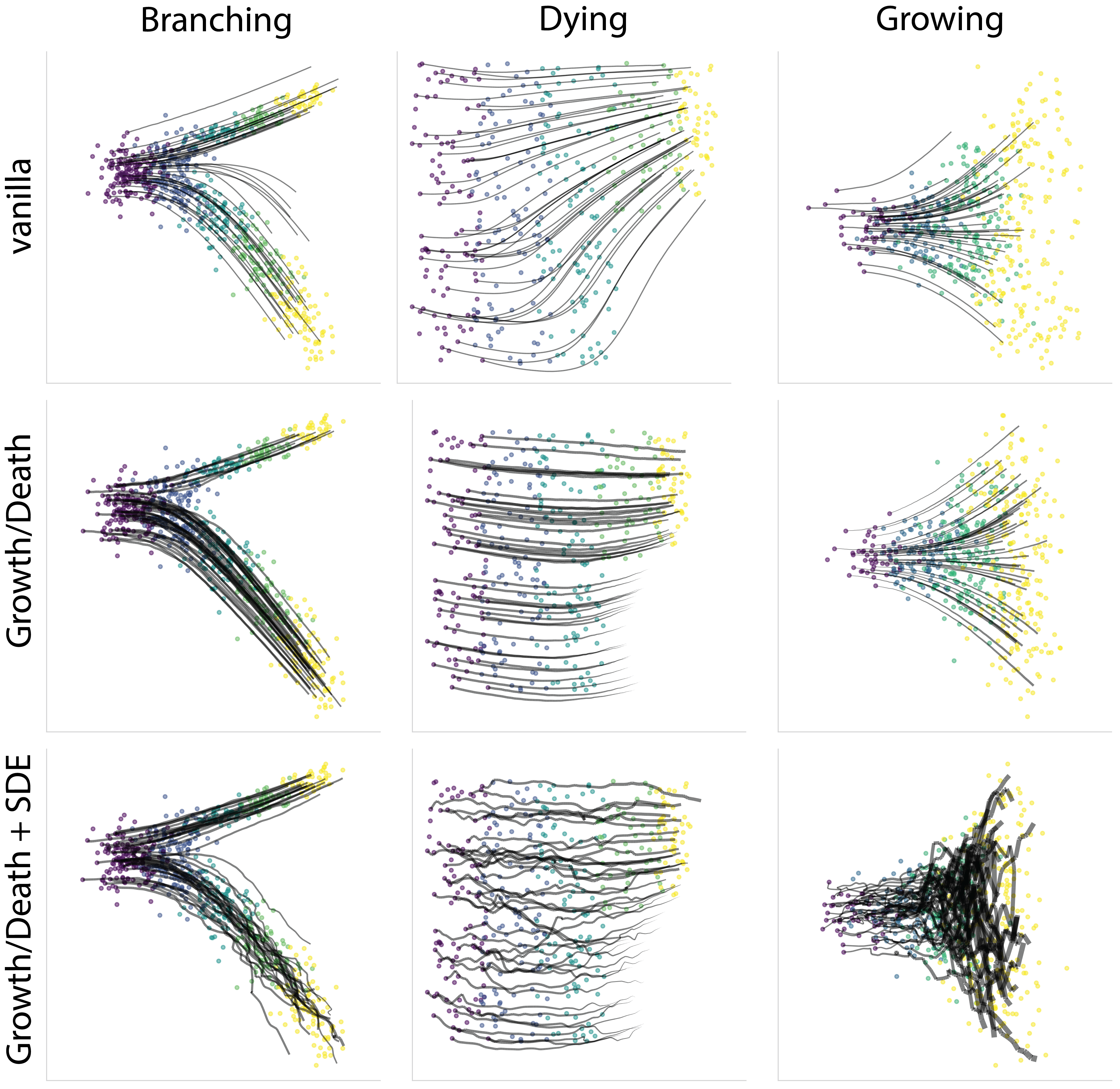}
  \caption{We evaluate MIOFlow 2.0 across three simulated datasets designed to mimic key biological characteristics: branching, population decline (dying), and proliferation (growing). By comparing the base MIOFlow 2.0 framework with variants incorporating the growth-rate model and Neural SDEs, we visualize the resulting trajectories. The results demonstrate that these integrated biological priors more accurately capture the geometry of the data manifold and the underlying dynamics compared to the base interpolation.}
  \label{fig:ablation}
\end{figure}

\subsection{Biological Case Study: Spatial-transcriptomics on axolotl data}

In order to validate the performance of the spatial-variant of MIOFlow 2.0 on real data, we applied it to a spatial transcriptomics dataset of axolotl brain samples. The dataset records spatial gene expression over seven timepoints of axolotl brain regeneration \cite{wei_axolotl}. The authors describe a regenerative trajectory traversing four sequential cell types - reaEGCs, rIPCs, IMNs, and nptxEXs, so we sought to interrogate spatial features involved in this particular developmental pathway.

Spatial features for cell type neighborhood, ligand-receptor expression, and local expression niche wwre computed as described above. The resulting concatenated spatial features were dimensionality reduced by PCA, and the dataset was subset to include only the cell types of interest. The resulting spatial feature vectors where used as input in the PHATE-regularized autoencoder to generate a spatial-feature embedding. This embedding was scaled by a factor of $s=0.2$ relative to the gene embedding, which together give a 4-dimensional embedding of the data, shown in Figure~\ref{fig:axolotl} A-C. Details pertaining to preprocessing and feature extraction are provided in Appendix~\ref{app:axolotl}.

The resulting trajectories are shown in Figure~\ref{fig:axolotl} D. These trajectories exhibit continuity in both the gene and spatial-feature embeddings, capturing smooth transitions between cell types occupying similar spatial niches. We are able to decode these trajectories to recover gene trends over time, shown in Figure~\ref{fig:axolotl} G, as well as dynamics in spatial features. Among these spatial trends, we observe an increase in NCAN:SDC3 signalling, (Figure~\ref{fig:axolotl} E-F), a lingand-receptor pair whose signalling activity has been implicated in promoting neurite outgrowth and regulating cell adhesion \cite{NCANSDC3}. The distribution of this spatial feature within the joint embedding space is shown in Figure~\ref{fig:axolotl} C. Only a certain subpopulation of terminal cell states experience high NCAN:SDC3 signalling, and MIOFlow 2.0 is able to distinguish which initial cell states putatively develop into this subpopulation Figure~\ref{fig:axolotl} D. Altogether, the capacity of MIOFlow 2.0 to incorporate and recover trends in spatial features enables a much richer analysis of spatail scRNA-seq data.

\begin{figure}
    \centering
    \includegraphics[width=0.9\linewidth]{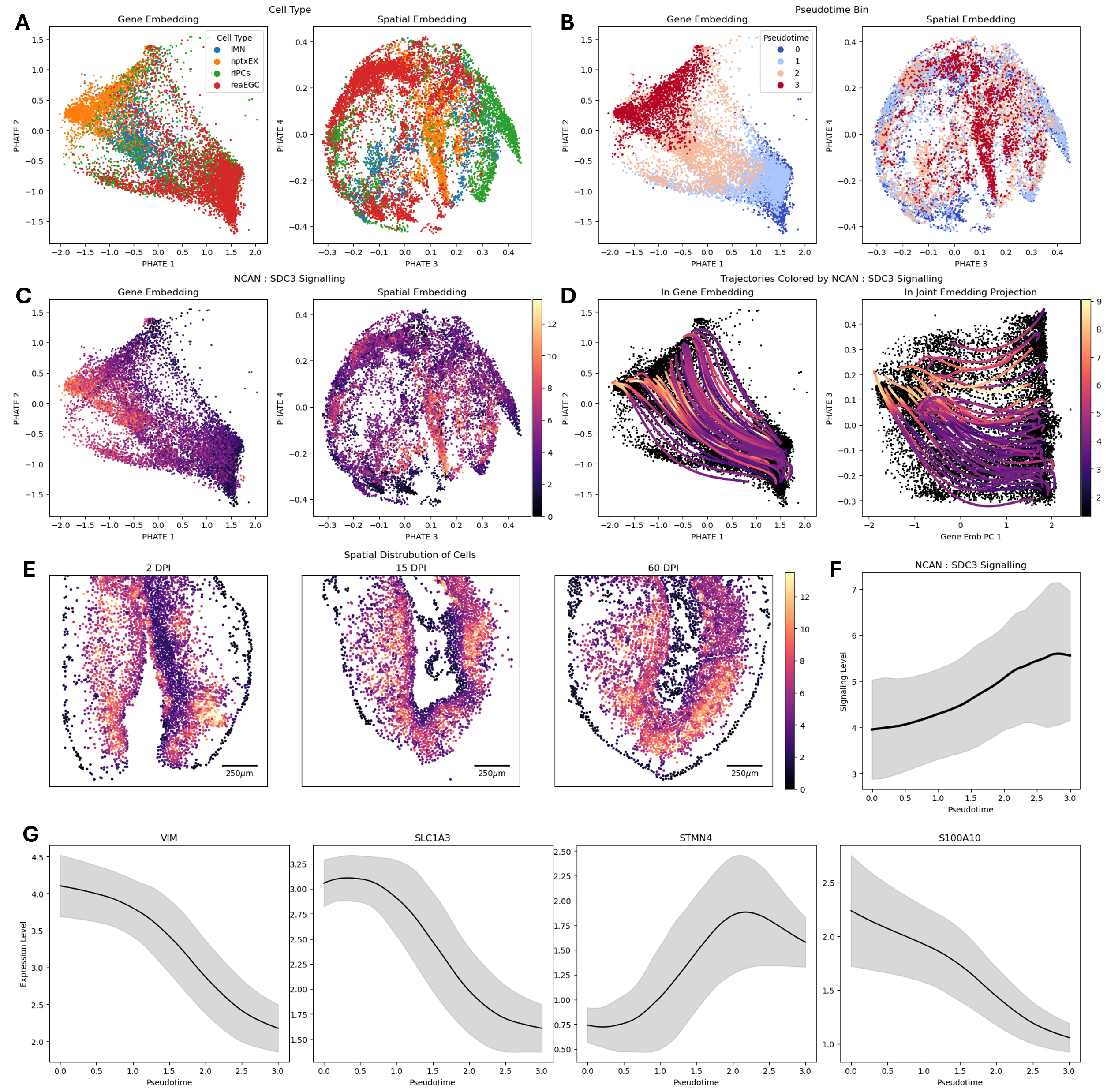}
    \caption{Gene and spatial feature embeddings colored by (A) cell type annotation, (B) pseudotime point. (C) NCAN:SDC3 ligand-receptor signalling spatial feature. (D) Trajectories traced on top of gene embedding (left) and joint projection of gene and spatial embeddings (right) colored by decoded NCAN:SDC3 signalling feature. (E) Spatial organization of cells in Stereo-seq dataset colored by NCAN:SDC3 signalling feature. Note that not all cells shown are included in the trajectories, but all are used for computing spatial features. (F) Decoded NCAN:SDC3 signalling feature over time, averaged across trajectories with +/- one standard deviation. (G) Decoded gene trajectories for a selection of highly variable genes.}
    \label{fig:axolotl}
\end{figure}

\section{Discussion}
\label{sec:discussion}

In this study, we developed MIOFlow 2.0, a comprehensive framework motivated by the the need for inferring biologically plausibile cellular trajectories from static snapshot single cell data.
Unlike existing methods that rely on simple Euclidean interpolations assume constant population densities \cite{lipman2023flow,tong2023improving}, MIOFlow 2.0 unifies manifold learning, dynamic optimal transport, and neural differential equations to model the complex realities of living systems.
The strength of our approach lies in the simultaneous integration of three essential biological priors: stochasticity, non-uniform proliferation, and spatial micro-environmental conditioning.

The first pillar of our framework, the integration of Neural Stochastic Differential Equations (Neural SDEs), allows the model to capture the inherent noise of gene expression and the probabilistic nature of cell fate decisions, which non-stochastic methods like flow matching do not allow.
As shown in our synthetic trifurcation experiments, the diffusion term in our SDE formulation enables the recovery of divergent lineages that deterministic ODE-based flows cannot naturally represent. By learning the diffusion coefficient alongside the drift, MIOFlow 2.0 identifies regions of the manifold where fate commitment is unresolved, providing a more accurate representation of differentiation hierarchies.

The second pillar addresses the non-conservative nature of biological populations through a learned growth and death rate model. With this technique, we solve a critical bottleneck in trajectory inference: the incorrect "mass-shifting" that occurs when traditional OT forces unnatural correspondence between differing celltypes. This allows the framework to faithfully model biological scenarios such as cancer treatment response or rapid embryonic expansion, where the expansion and attrition of specific sub-populations are central to the underlying dynamics. Indeed, we have recently applied a simplified version of this framework to capture tumorsphere development~\citep{cflows}, where the model successfully reconstructed trajectories leading to either tumorsphere formation or apoptosis, delineating a novel CD44$^{hi}$EPCAM$^{+}$CAV1$^{+}$ cancer stem cell profile.

The third pillar is the spatial integration strategy, which enables trajectory inference to leverage the rich contextual information provided by spatial transcriptomics. By embedding the microenvironmental cues directly into the latent state space, we allow the model to learn dynamics that are a function of the total cellular state. Our results on axolotl brain regeneration demonstrate that this joint modeling captures dependencies that gene expression alone cannot.
This bridges the gap between single-cell profiling and tissue-scale organization, allowing us to identify signaling niches that drive development and disease.

From a computational perspective, MIOFlow 2.0 demonstrates that expressive dynamics learning via NSDEs outperforms popular generative models like flow matching in matching biological trajectories. By operating in a latent space of a PHATE-regularized autoencoder, we ensure that the transport energy is minimized along the data's intrinsic manifold geometry. The stabilization provided by our momentum-based refinement further ensures that trajectories remain consistent even when temporal data is sparse.  In conclusion, MIOFlow 2.0 provides a robust, biologically-aware foundation for understanding how cells navigate complex developmental, regenerative, and pathological landscapes.Future work will focus on scaling these solvers to handle larger multi-modal datasets and incorporating higher-order graph topologies to further refine the spatial context.

Finally, we distinguish our computational approach from experimental lineage tracing~\citep{mckenna2016whole,spanjaard2018simultaneous,alemany2018whole}. While lineage tracing technologies are developing, they remain relatively crude for mapping dynamics: most protocols accumulate only a small number of noisy barcodes, yielding discrete snapshots of clonal ancestry rather than continuous trajectories. These methods can identify clonal endpoints, but they cannot resolve the continuous, state-specific transitions that drive these outcomes. MIOFlow 2.0 fills this gap by inferring the dense, continuous vector fields that sparse lineage data cannot capture.


\bibliographystyle{ACM-Reference-Format}
\bibliography{ref}

\appendix
\newpage
\section{Proof of Theorem 1}
\label{app:proof}

\begin{customthm}{1}
Consider a time-varying vector field $f(z,t)$ defining latent cellular trajectories $dZ_{u,t} = f(Z_{u,t},t)dt$ with instantaneous density $\rho_t$, and a dissimilarity metric $D(\mu,\nu)$ such that $D(\mu,\nu)=0$ iff $\mu=\nu$.
Given these assumptions, there exists a sufficiently large regularization parameter $\lambda > 0$ such that the optimal transport problem satisfies:
$$W_2(\mu,\nu)^2 = \inf_{Z_{u,t}} \mathbb{E}\bigg[ \int_0^1 \|f(Z_{u,t},t)\|_2^2 dt \bigg] + \lambda D(\rho_1,\nu), \quad \text{s.t. } Z_{u,0} \sim \mu.$$
Moreover, because the process $Z_{u,t}$ is defined on the embedded manifold space $\mathcal{Z}$ learned by our geometry-aware autoencoder, the Euclidean Wasserstein distance in latent space approximates the geodesic Wasserstein distance on the ambient manifold: $W_2(\mu,\nu) \simeq W_{d_{\mathcal{M}}}(\mu,\nu)$.
\end{customthm}

\begin{proof}
We recall that the exact dynamic formulation
$$W_2(\mu,\nu)^2 = \inf_{Z_{u,t}} \mathbb{E}\bigg[ \int_0^1 \|f(Z_{u,t},t)\|_2^2 dt \bigg] \, \text{ s.t. }dZ_{u,t} = f(Z_{u,t},t)dt,\, Z_{u,0}\sim\mu,\,Z_{u,1}\sim\nu,$$
is equivalent to the Eulerian representation
$$W_2(\mu, \nu)^2 = \inf_{(\rho_t,v)}\int_{0}^{1}\int_{\real^d} \|v(z,t)\|^2 \rho_t(dz) dt,$$
with the three constraints a) $\partial_t \rho_t + \nabla \cdot (\rho_t v) = 0$, b) $\rho_0 = \mu$, and c) $\rho_1=\nu$.
\citet{tong_trajectorynet_2020} showed that for a large $\lambda>0$ and $\rho_t$ satisfying constraint a, this minimization problem is equivalent to the relaxed form
$$W_2(\mu, \nu)^2 = \inf_{(\rho_t,v)}\int_{0}^{1}\int_{\real^d} \|v(z,t)\|^2 \rho_t(dz) dt + \lambda KL(\rho_1\,||\nu).$$
We note that their proof is valid for any dissimilarity metric $D(\rho_1,\nu)$ respecting the identity of indiscernibles.
Using the path formulation, by writing the integral as an expectation and taking the infimum over all absolutely continuous paths, we have
$$W_2(\mu,\nu)^2 =\inf_{Z_{u,t}} \mathbb{E}\bigg[ \int_0^1 \|f(Z_{u,t},t)\|_2^2 dt \bigg] + \lambda D(\rho_1,\nu) \, \text{ s.t. }\,dZ_{u,t} = f(Z_{u,t},t)dt,\, Z_{u,0}\sim\mu.$$
Assuming the geometry-aware encoder $E_\phi$ achieves a perfect mapping, the Euclidean distance in the latent space approximates the geodesic distance on the ambient manifold.
This means $\|E_\phi(x_i)-E_\phi(x_j)\|_2 \simeq d_\M(x_i,x_j)$ for all $x_i,x_j\in\X\subseteq\M$.
Therefore, there exist constants $c,C>0$ such that $c \,d_\M(x_i,x_j) \leq \|E_\phi(x_i)-E_\phi(x_j)\|_2 \leq C d_\M(x_i,x_j)$ for all $x_i,x_j\in \X$.
Then, for all joint distributions $\pi\in\Pi(\mu_{ambient},\nu_{ambient})$, we have
$$c^2\int_{\X\times\X} d_\M(x,y)^2 \pi(dx,dy)\leq \int_{\X\times\X} \|E_\phi(x)-E_\phi(y)\|_2^2\pi(dx,dy)\leq C^2\int_{\X\times\X}d_\M(x,y)^2\pi(dx,dy).$$
Taking the infimum with respect to $\pi$ yields the desired result $W_2(\mu,\nu)\simeq W_{d_\mathcal{M}}(\mu_{ambient},\nu_{ambient})$.
\end{proof}

\section{Full MIOFlow 2.0 Algorithm}
\label{app:full_mioflow}
\begin{algorithm}[H]
\caption{MIOFlow (Global training)}
\label{alg:mioflow_global}
\begin{algorithmic}[1]
\State Input:
\begin{itemize}
    \item Time-resolved cell-by-gene expression matrices $X_t = (x_{cgt})_{c\in\mathcal{C}_t,\,g\in\mathcal{G}},\; t=1,\dots,T$
    \item (Optional) spatial features $S_t = (s_{cit})_{c\in\mathcal{C}_t,\,i=1,\dots,p},\; t=1,\dots,T$
    \item Solver: $\texttt{solver} \in \{\text{ODE}, \text{SDE}\}$
    \item Momentum parameter: $\beta$
    \item Loss weights $\lambda_m,\lambda_e,\lambda_d$
\end{itemize}
\State Output: Trained neural network weights $\theta,\phi,\psi$
\For{$ t = 1 $ to $ T $}
  \If{spatial features provided}
    \State $ X_{joint,t} \gets \texttt{Concat}(\texttt{Normalize}(X_t), \texttt{Normalize}(S_t)) $
  \Else
    \State $ X_{joint,t} \gets X_t $
  \EndIf
  \State $ Z_t \gets \texttt{GAGA\_Encoder}(X_{joint,t}) $ \Comment{Embed using PHATE-regularized autoencoder}
\EndFor
\For{$ i = 1 $ to $\text{max\_iter}$}
  \For{$ t = 1 $ to $ T $}
    \State $ z_t \gets \texttt{Sample}(Z_t, \text{batch\_size}) $
  \EndFor
  \If{$\texttt{solver} = \text{ODE}$}
    \State Set $ g_{\phi}(\cdot) \gets 0 $ \Comment{ODE: diffusion term not used}
  \EndIf
  \State \Comment{Global training: Integrate from the first time point to the rest.}
  \State $ \hat{z}_{2} \gets \texttt{DiffEqSolve}\Big(\tilde f_\theta(\cdot,\beta), \tilde g_\phi(\cdot), z_1, \Delta t=1\Big) $
  \State $ \hat{m}_{2} \gets \texttt{Normalize}\Big(h_{\psi}(\hat{z}_{2},2)\Big) $
  \State $ m_{2} \gets \texttt{Normalize}((1,\dots,1)^T) $
  \For{$ t = 2 $ to $ T-1 $}
    \State $ \hat{z}_{t+1} \gets \texttt{DiffEqSolve}\Big(\tilde f_\theta(\cdot,\beta), \tilde g_{\phi}(\cdot), \hat{z}_{t}, \Delta t=1\Big) $
    \State $ \hat{m}_{t+1} \gets \texttt{Normalize}\Big(h_{\psi}(\hat{z}_{t+1},t+1)\Big) $
    \State $ m_{t+1} \gets \texttt{Normalize}((1,\dots,1)^T) $
  \EndFor
  \State $ L \gets \sum_{t=1}^{T-1} \Big[ \lambda_m L_m\big(\hat{z}_{t+1}, z_{t+1}, \hat{m}_{t+1}, m_{t+1}\big) + \lambda_e L_e\big(\hat{z}_{t+1}, z_{t+1}\big) + \lambda_d L_d\big(\hat{z}_{t+1}, z_{t+1}\big) \Big] $
  \State $ (\theta,\phi,\psi) \gets \texttt{GradientDescent}(L, \theta,\phi,\psi) $
\EndFor
\State \Return $ \theta,\phi,\psi $
\end{algorithmic}
\end{algorithm}

\begin{algorithm}[H]
\caption{MIOFlow (Local training)}
\label{alg:mioflow_local}
\begin{algorithmic}[1]
\State Input:
\begin{itemize}
    \item Time-resolved cell-by-gene expression matrices $X_t = (x_{cgt})_{c\in\mathcal{C}_t,\,g\in\mathcal{G}},\; t=1,\dots,T$
    \item (Optional) spatial features $S_t = (s_{cit})_{c\in\mathcal{C}_t,\,i=1,\dots,p},\; t=1,\dots,T$
    \item Solver: $\texttt{solver} \in \{\text{ODE}, \text{SDE}\}$
    \item Momentum parameter: $\beta$
    \item Loss weights $\lambda_m,\lambda_e,\lambda_d$
\end{itemize}
\State Output: Trained neural network weights $\theta,\phi,\psi$
\For{$ t = 1 $ to $ T $}
  \If{spatial features provided}
    \State $ X_{joint,t} \gets \texttt{Concat}(\texttt{Normalize}(X_t), \texttt{Normalize}(S_t)) $
  \Else
    \State $ X_{joint,t} \gets X_t $
  \EndIf
  \State $ Z_t \gets \texttt{GAGA\_Encoder}(X_{joint,t}) $ \Comment{Embed using PHATE-regularized autoencoder}
\EndFor
\For{$ i = 1 $ to $\text{max\_iter}$}
  \For{$ t = 1 $ to $ T $}
    \State $ z_t \gets \texttt{Sample}(Z_t, \text{batch\_size}) $
  \EndFor
  \If{$\texttt{solver} = \text{ODE}$}
    \State Set $ g_{\phi}(\cdot) \gets 0 $ \Comment{ODE: diffusion term not used}
  \EndIf
  \State \Comment{Local training: Integrate from each time point to the next.}
  \For{$ t = 1 $ to $ T-1 $}
    \State $ \hat{z}_{t+1} \gets \texttt{DiffEqSolve}\Big(\tilde f_\theta(\cdot,\beta), \tilde g_{\phi}(\cdot), z_t, \Delta t=1\Big) $
    \State $ \hat{m}_{t+1} \gets \texttt{Normalize}\Big(h_{\psi}(\hat{z}_{t+1},t+1)\Big) $
    \State $ m_{t+1} \gets \texttt{Normalize}((1,\dots,1)^T) $
  \EndFor
  \State $ L \gets \sum_{t=1}^{T-1} \Big[ \lambda_m L_m\big(\hat{z}_{t+1}, z_{t+1}, \hat{m}_{t+1}, m_{t+1}\big) + \lambda_e L_e\big(\hat{z}_{t+1}, z_{t+1}\big) + \lambda_d L_d\big(\hat{z}_{t+1}, z_{t+1}\big) \Big] $
  \State $ (\theta,\phi,\psi) \gets \texttt{GradientDescent}(L, \theta,\phi,\psi) $
\EndFor
\State \Return $ \theta,\phi,\psi $
\end{algorithmic}
\end{algorithm}

\section{Axolotl Data Processing}
\label{app:axolotl}
For preprocessing, we applied Harmony batch integration to remove some strong batch effects between timepoints \cite{Harmony}. To compute spatial features, we used a 3-hop 5-nearest-neighbor graph to define each cell's neighborhood, excluding any nearest-neighbors with a distance greater than 250$\mu m$. The cell type neighborhood feature quantified the local frequency of 28 annotated cell types. For the ligand-receptor features, we found 640 pairs of ligand-receptor interactors documented in the human CellChat database which had orthologues expressed in the axolotl data \cite{jin_cellchat}. Prior to computing the ligand-receptor features, we imputed the expression of all relevant genes using MAGIC \cite{magic}. The local expression niche was computed as the local average of the first 200 principal components of the gene expression matrix. The concatenated 828-dimensional spatial feature vector was then dimensionality reduced to 100-d by PCA. These features, subset to the four cell types of interest, were dimensionality reduced by PHATE and used as input to the PHATE regularized autoencoder. On the other hand, the gene embedding was obtained by computing a PHATE embedding of the full dataset's gene expression matrix, which was then subset to the cell types of interest and used as input to a separate PHATE regularized autoencoder, along with the 200-d PCA projection of the gene expression matrix. To obtain the final joint embedding, the gene and spatail embeddings were scaled to have standard deviations of 1.0 and 0.2 respectively and concatenated. To define the time bins for MIOFlow 2.0, we used scanpy's implementation of diffusion pseudotime, averageing pseudotime values computed from multiple root cells of the final cell type, inverting, and binning into four levels \cite{scanpy}.

\subsection{Synthetic single-cell data generation with SERGIO}\label{app:sergio}

We use SERGIO \cite{SERGIO}, a stochastic gene regulatory network (GRN) simulator, to generate synthetic single-cell gene expression data with known ground-truth regulatory structure, differentiation trajectories, and temporal dynamics.

SERGIO begins by constructing a directed gene regulatory network consisting of $G$ genes, where edges represent regulatory interactions. A subset of genes are designated as master regulators, which have no incoming edges and act as external drivers of cellular programs. The remaining genes are regulated by one or more upstream transcription factors according to the specified GRN topology. Each regulatory edge $(j \rightarrow i)$ is associated with a signed interaction strength $w_{ij}$, indicating activation or repression. Cell type specific expression programs are defined by prescribing target expression levels for the master regulators, which serve as boundary conditions for different cellular states or fates.

Gene expression dynamics are simulated by integrating a system of stochastic differential equations forward in time. For each gene $i$, the expression level $x_i(t)$ evolves according to
$$\frac{d x_i(t)}{dt} = b_i + \sum_{j \in \mathrm{reg}(i)} w_{ij} \frac{x_j(t)^{n_{ij}}}{K_{ij}^{n_{ij}} + x_j(t)^{n_{ij}}} - \lambda_i x_i(t) + \eta_i(t)$$
where $b_i$ is the basal transcription rate, $\mathrm{reg}(i)$ denotes the set of regulators of gene $i$, $K_{ij}$ and $n_{ij}$ are the Hill constant and cooperativity coefficient, respectively, $\lambda_i$ is a gene-specific degradation rate, and $\eta_i(t)$ is a stochastic noise term modeling intrinsic transcriptional variability. The noise is typically modeled as Gaussian with variance proportional to expression level, generating realistic cell-to-cell heterogeneity.

To simulate differentiation and lineage branching, SERGIO defines a directed graph over discrete cell states and integrates the stochastic dynamics along each transition. Continuous gene expression trajectories are generated by gradually interpolating master regulator expression programs between connected states. Single cells are then obtained by sampling expression profiles along these trajectories, producing populations that capture progenitor states, intermediate transitions, and terminal fates.

To emulate experimental single-cell RNA sequencing data, SERGIO applies a technical noise model to the simulated expression values, including library size effects, dropout events, and sampling noise, resulting in sparse count matrices. Optionally, SERGIO further decomposes expression levels into spliced and unspliced mRNA counts by simulating transcription, splicing, and degradation processes, enabling downstream RNA velocity analyses. We generated two synthetic datasets representing trajectory topologies commonly observed in single-cell studies: a trifurcating differentiation process and a curved, S-shaped trajectory.

\subsubsection{Trifurcation trajectories}

The trifurcation dataset simulates a differentiation process in which a single progenitor cell state gives rise to three distinct terminal cell fates. We simulated 100 genes across 500 cells, organized according to a differentiation graph with one initial cell type and three downstream cell types.  This was achieved by defining a differentiation graph in which one initial cell type has outgoing transitions to three downstream cell types, each associated with a distinct master regulator expression program. Cells were sampled along all three trajectories, producing a continuous branching structure with a shared progenitor region and three diverging lineages.

\subsubsection{S-shaped trajectory}

The S-shaped dataset models a more complex trajectory involving cell-cycle dynamics coupled with fate specification. We simulated a cell-cycle driven differentiation process consisting of 1000 genes and 990 cells. In this setting, cells first undergo a cyclic progression corresponding to cell-cycle dynamics. From a specific point along the cycle, the trajectory bifurcates into two terminal fates. The final S-shaped dataset was obtained by subsampling 315 cells along the combined cycle and bifurcation trajectories, resulting in a curved differentiation path.

\section{Distinction from Spatially Informed Trajectory Methods}
\label{app:comparison}

Several trajectory-inference methods have recently been proposed that incorporate spatial coordinates from spatial transcriptomics data, including SpaTrack \cite{SHEN2025101194} and NicheFlow \cite{sakalyan2025modeling}. While these methods share with MIOFlow 2.0 the goal of leveraging spatial information, they differ fundamentally in how these features are represented and incorporated into the trajectories.

One crucial advantage of MIOFlow 2.0 is its use of biologically relevant spatial features as opposed to explicit spatial coordinates. Each method above attempts to learn trajectories between timepoints informed by both gene expression and the cell/neighborhood's spatial position. The underlying assumption of these methods is that each timepoint shares an implicit underlying coordinate system which would allow the timepoints to be aligned spatially, such that cells in the dorsal and anterior region of the sample, for example, remain there. This assumption is not valid for most biological applications - cells in culture or many tissue samples lack any kind of directionality which would define a meaningful alignment. Only MIOFlow 2.0, which allows comparisons of spatial features between timepoints independent of any shared coordinate system, is able to leverage the spatial information of these datasets.

SpaTrack learns optimal transport maps between consecutively observed timepoints, incorporating both gene expression and spatial distance into the transport cost. However, SpaTrack's interpolation is obtained by composing discrete transport maps under a Markov assumption, whereas MIOFlow 2.0 learns a continuous flow via a neural ODE, producing trajectories shaped by the learned dynamics rather than by composition of pairwise couplings. Furthermore, SpaTrack relies on explicit spatial coordinates, requiring a shared coordinate system across timepoints, while MIOFlow 2.0 operates on derived spatial features that are coordinate-system independent.

NicheFlow and MIOFlow 2.0 are the only two methods that make use of the gene expression of the cells' local neighborhoods, a biologically relevant property informing a cell's fate. NicheFlow, however, differs from MIOFlow 2.0 in that it operates at the level of cellular microenvironments rather than individual cells. It models the evolution of local neighborhoods as point clouds, jointly predicting changes in spatial coordinates and gene expression at the niche level. While this captures coordinated tissue-level dynamics, it does not produce individual cell trajectories. Moreover, NicheFlow does not learn a direct flow between observed timepoints. Rather, it learns a conditional generative model: given a microenvironment at time $t$, it generates a corresponding neighborhood at time $t+1$ by learning a flow from Gaussian noise conditioned on the source neighborhood. This means that the learned dynamics do not describe a continuous transformation of cells through time, but rather a conditional sampling procedure that produces plausible future neighborhoods. In contrast, MIOFlow 2.0 learns a continuous flow in which cells traverse an interpretable trajectory from their initial to their final state, enabling meaningful interpolation and the recovery of intermediate dynamics.




\end{document}